\documentclass[11pt]{article}
\usepackage[margin=1in]{geometry}

\usepackage{tikz}

\usepackage[round]{natbib}
\usepackage{amsfonts}       
\usepackage{nicefrac}       
\usepackage{amsmath,amssymb,amsthm}

\usepackage[colorlinks,allcolors=blue]{hyperref}

\newcommand\nnfootnote[2]{%
  \begin{NoHyper}
  \renewcommand\thefootnote{#1}\footnotetext{#2}%
  \addtocounter{footnote}{-1}%
  \end{NoHyper}
}

\usepackage{algorithmic,algorithm}
\usepackage{float}
\usepackage{mathtools,thmtools}
\usepackage[title]{appendix}
\usepackage{ dsfont }
\usepackage[capitalize]{cleveref}
\usepackage{ifthen}
\usepackage{xcolor}

\theoremstyle{plain}
\newtheorem{theorem}{Theorem}
\newtheorem{lemma}{Lemma}
\newtheorem{claim}{Claim}
\newtheorem{corollary}{Corollary}

\declaretheoremstyle[ 
        spaceabove=\topsep, 
        spacebelow=\topsep, 
        headfont=\normalfont\itshape,
        bodyfont=\normalfont,
        notefont=\normalfont\itshape,
        notebraces={}{},
        postheadspace=0.33em, 
        qed=$\square$, 
        headpunct={.},
    ]{proofstyle}
\declaretheorem[style=proofstyle,numbered=no,name=Proof]{proof}

\newcommand{\R}{\mathbb{R}}

\newcommand{\calA}{\mathcal{A}}

\newcommand{\eps}{\varepsilon}

\newcommand{\domifulli}{\Delta_{\calA_i}}

\renewcommand{\S}{\mathcal{S}}

\newcommand{\A}{\mathcal{A}}

\newcommand{\ba}{\mathbf{a}}
\newcommand{\bs}{\mathbf{s}}
\newcommand{\bpi}{\boldsymbol{\pi}}

\newcommand{\X}{\mathcal{X}}
\newcommand{\C}{\mathcal{C}}

\newcommand{\reg}{R}
\renewcommand{\Re}{\mathfrak{R}}

\DeclarePairedDelimiterX{\infdiv}[2]{(}{)}{%
  #1\;\delimsize\|\;#2%
}

\newcommand{\T}{^\top}
\newcommand{\eqq}{\coloneqq}

\newcommand{\dotprod}[2]{\left\langle #1 , #2 \right\rangle}

\DeclareMathOperator*{\argmin}{arg\,min}

\DeclareMathOperator{\E}{\mathbb{E}}

\newcommand{\norm}[1]{\left\|#1\right\|}
\newcommand{\norms}[1]{\|#1\|}
\newcommand{\normb}[1]{\big\|#1\big\|}

\newcommand{\brackets}[4][a]{%
  \ifx a#1\left#2 #3 \right#4\else
  \ifx s#1#2 #3 #4\else
  \ifx b#1\big#2 #3 \big#4\else
  \ifthenelse{\equal{#1}{B}} {\Big#2 #3 \Big#4} {}
  \ifthenelse{\equal{#1}{Bg}} {\Bigg#2 #3 \Bigg#4} {}
  \fi\fi\fi
}

\renewcommand{\b}[2][a]{\brackets[#1]{(}{#2}{)}}
\newcommand{\cb}[2][a]{\brackets[#1]{\{}{#2}{\}}}
\renewcommand{\sb}[2][a]{\brackets[#1]{[}{#2}{]}}
\newcommand{\ab}[2][a]{\brackets[#1]{\langle}{#2}{\rangle}}
\newcommand{\av}[2][a]{\brackets[#1]{|}{#2}{|}}

\newcommand{\SwRe}{{\rm Swap}\Re}

\renewcommand{\l}{\ell}

\usepackage{soul}

\hypersetup{
           breaklinks=true,   
}

\title{Regret Minimization and Convergence to Equilibria \\ in General-sum Markov Games}

\author{
Liad Erez$^{1, *}$
\and
Tal Lancewicki$^{1, *}$
\and
Uri Sherman$^{1, *}$
\and
Tomer Koren$^{1,2}$
\and
Yishay Mansour$^{1,2}$
}

\begin{document}
\maketitle

\nnfootnote{1 }{ Blavatnik School of Computer Science, Tel Aviv University, Tel Aviv, Israel.}
\nnfootnote{2 }{ Google Research, Tel Aviv.}

\nnfootnote{* }{ Equal contribution.}

\begin{abstract}
    An abundance of recent impossibility results establish that regret minimization in Markov games with adversarial opponents is both statistically and computationally intractable.
    Nevertheless, none of these results preclude the possibility of regret minimization under the assumption that all parties adopt the same learning procedure.
    In this work, we present the first (to our knowledge) algorithm for learning in general-sum Markov games that provides sublinear regret guarantees when executed by all agents.
    The bounds we obtain are for \emph{swap regret}, 
    and thus, along the way, imply convergence to a \emph{correlated} equilibrium.
    Our algorithm is decentralized, computationally efficient, and does not require any communication between agents.
    Our key observation is that
    online learning via policy optimization in Markov games essentially reduces to a form of \emph{weighted} regret minimization, with \emph{unknown} weights determined by the path length of the agents' policy sequence. Consequently, controlling the path length leads to weighted regret objectives for which sufficiently adaptive algorithms provide sublinear regret guarantees.

\end{abstract}

\section{Introduction}
Multiagent reinforcement learning (MARL; see \citealp{busoniu2008comprehensive, zhang2021multi}) studies statistical and computational properties of learning setups that consist of multiple agents interacting within a dynamic environment.
One of the most well studied models for MARL is Markov Games (also known as \emph{stochastic games}, introduced originally by \citealp{shapley1953stochastic}), which can be seen as a generalization of a Markov Decision Process (MDP) to the multiagent setup.
In this model, the transition dynamics are governed by the \emph{joint} action profile of all agents, 
implying that the environment as perceived by any individual agent is non-stationary. While providing powerful modeling capabilities, this comes at the cost of marked challenges in algorithm design.
Furthermore, in its full generality the model considers multiplayer \emph{general-sum} games,
where it is well-known that computing a Nash equilibrium is computationally intractable already in the simpler model of normal form games \citep{daskalakis2009complexity, chen2009settling}.

Contemporary research works that study general-sum Markov games consider objectives that roughly fall into one of two categories; sample complexity of learning an approximate (coarse) correlated equilibrium, or regret against an arbitrary opponent. The sample complexity setup assumes all players learn using the same algorithm, while in the regret minimization setting, where the vast majority of results are negative \cite[e.g.,][]{bai2020near, tian2021online, liu2022learning}, the opponents are assumed to be adversarial, and in particular do not use the same algorithm as the learner nor attempt to minimize their regret. 
Curiously, developing (or, asking if there exist) algorithms that minimize individual regret given that all players adopt the same algorithm has been largely overlooked. 
Considering the intrinsic nature of MARL problems, where agents learn interactively from experience, it is of fundamental interest not only to arrive at an equilibrium, but to control the loss incurred during the learning process.
Moreover, this is precisely the objective considered by a long line of works into learning in normal form games \citep{syrgkanis2015fast,chen2020hedging, daskalakis2021near, anagnostides2022uncoupled}.
Thus, we are motivated to ask;

\begin{center}
\bfseries\itshape
Can we design algorithms for learning in general-sum Markov games that, \\
when adopted by all agents, provide sublinear individual regret guarantees?
\end{center}

In this work, we answer the above question affirmatively, and present the first (to our knowledge) algorithm for general-sum Markov games with such a sublinear regret guarantee.
We consider finite horizon Markov games in the full-information setting, where access to exact state-action value functions is available. Our algorithm is decentralized and does not require any form of communication between agents. In addition, our bounds apply to the general notion of \emph{swap regret} \citep{blum2007external}, and therefore imply that the empirical distribution (over time steps) of policy profiles generated by our algorithm converges to a correlated equilibrium as the game progresses.

To achieve our results, we make the following observations. In a Markov game, from the point of view of any individual agent, the environment reduces to a single agent MDP in any given episode.
When considering multiple episodes, 
the environment perceived by any individual agent is non-stationary, with the path length of the sequence of policies generated by fellow agents determining the total variation of MDP dynamics.
Our first key observation is that, when executing a policy optimization routine \citep{shani2020optimistic, cai2020provably} in a non-stationary MDP (and thus in a Markov game), 
the per state objective becomes one of \emph{weighted regret} with weights \emph{unknown} to the learner. Importantly,
the total variation of weights in these objectives is governed by the degree of non-stationarity (and in turn, by the path length of the other agents' policies).
Therefore, a possible approach would be to provide all agents with an algorithm which has the following two properties; (1) the path length of generated policies is well bounded, and (2) the per state weighted regret is bounded in terms of the total variation of weights (and thus in terms of the policy path length).
Indeed, we prove that a carefully designed instantiation of policy optimization with \emph{optimistic-online-mirror-descent} (OOMD; \citealp{rakhlin2013optimization}) produces a bounded path length policy sequence, and simultaneously exhibits the required weighted regret bounds.

Our approach builds on recent progress on decentralized learning in normal form games \citep{daskalakis2021near, anagnostides2022uncoupled, anagnostides2022faster}. The work of \cite{anagnostides2022uncoupled} demonstrated that \emph{optimistic-follow-the-regularized-leader} (OFTRL; \citealp{syrgkanis2015fast}), combined with log-barrier regularization and the no-swap-regret meta algorithm of \cite{blum2007external}, leads to well bounded path length in general-sum normal form games.
However, their techniques do not readily extend to the Markov game setup; indeed, FTRL-based algorithms are not sufficiently adaptive and at least in standard form cannot be tuned to satisfy weighted regret bounds. In fact, weighted regret is a generalization of the previously studied objective of \emph{adaptive regret} \citep{hazan2009efficient}, and it can be shown FTRL-based algorithms do not even satisfy this weaker notion (see \cite{hazan2009efficient} and a more elaborate discussion in \cref{sec:ftrl_nonstationary_lb}).
Evidently, however, an OOMD-based algorithm can be made sufficiently adaptive \emph{and} produce iterates of bounded path length.
When all agents adopt our proposed algorithm, the path length of the generated policy sequence remains well bounded, leading to moderate non-stationarity and low total variation per state weighted regret problems, which allows properly tuned mirror descent steps---crucially, \emph{without} knowledge of the weights---to obtain sublinear regret.
Notably, while much of the previous works \citep[e.g., ][]{syrgkanis2015fast, chen2020hedging, anagnostides2022uncoupled} employ optimistic online algorithms and path length dependent regret bounds to improve upon naive square-root regret, in Markov games, with our approach, these are actually crucial for obtaining any form of sublinear regret.

\paragraph{Addendum.}

Following the initial publication of this work in ICML'23~\citep{erez2023regret}, an error in the arguments given there was brought to our attention by Khashayar Gatmiry and Noah Golowich via private correspondence.
Specifically, the swap regret bound derived in the proof of \cref{thm:apprx_state_master_swap_rvu} relies on a benchmark vector that lies outside the decision set of the algorithm, rendering the argument invalid.
Despite our best efforts, we were unable to identify a fix to this issue that avoids imposing restrictive assumptions on the structure of the Markov game or introducing substantial modifications to our technical framework.
Consequently, the current version of the manuscript presents a corrected version that retains the majority of the original technical arguments, at the cost of an additional, execution dependent term in the final regret bound which cannot be bounded in the general case.

In some more detail, the new version of the algorithm operates over the full (non-shrunk) simplex, thereby allowing for the swap regret argument of \cref{thm:apprx_state_master_swap_rvu} to go through.
As a result, the maximal Bregman divergence encountered during the algorithm's execution appears in the final bound, which may lead to a vacuous (linear) regret bound without further assumptions.

\subsection{Summary of contributions}

To summarize, we present a decentralized algorithm (\cref{alg:POSR}) for (multiplayer, tabular, and episodic) Markov games, with the following guarantees when adopted by all agents.
Some of the results below are obtained under an additional technical condition involving the iterates of the algorithm, as discussed earlier.
\begin{itemize}
    \item In our main problem setup, with access granted to exact state-action value functions, the individual swap regret of every agent is $\widetilde O(T^{1/2+\alpha})$ (see \cref{sec:main}). 

    \item In the special case of full-information independent transition function where agents only affect the loss functions of each other but not the transition dynamics, our algorithm guarantees $O(\log T)$ individual regret. 
    The result is relatively straightforward given our analysis for general Markov games, and we defer the formal setting and proofs to \cref{sec:quasi_markov}. 
    
    \item As an immediate implication, we obtain that the joint empirical distribution of policy profiles produced by our algorithm converges to correlated equilibria, at a rate of $\widetilde O(T^{-1/2 + \alpha})$, and $\widetilde O(1/T)$ in the independent transition function setup. 
    
\end{itemize}

\subsection{Related work}

\paragraph{Learning in Markov games.}
The framework of Markov games was originally introduced by \cite{shapley1953stochastic}. The majority of studies consider learning Nash equilibria in two-player zero-sum Markov games, and may be roughly categorized by assumptions made on the model.
The full-information setting, where the transition function is known and/or some sort of minimum state reachability is assumed, has gained much of the earlier attention \citep{littman1994markov, littman2001friend, brafman2002r,hu2003nash, hansen2013strategy, wei2017online}, as well as more recent \citep{daskalakis2020independent, wei2021last, cen2021fast, zhao2022provably, alacaoglu2022natural, zhang2022policy}. 
The unknown model setup, where the burden of exploration is entirely in the hands of the agent, has been a target of several recent papers focusing on sample complexity \citep{sidford2020solving, bai2020near, xie2020learning, zhang2020model, liu2021sharp}.

The work of \cite{wei2021last} considers zero-sum games with model assumptions similar to ours, and present an optimistic gradient descent-ascent policy optimization algorithm with a smoothly moving critic. They obtain last iterate convergence to a Nash equilibrium at a rate of $\widetilde O(T^{-1/2})$ for the full-information setting, which immediately implies individual regret of $\widetilde O(T^{1/2})$. In the unknown model setup with reachability assumptions, their algorithm obtains a last iterate guarantee that implies $\widetilde O(T^{7/8})$ regret.
Also noteworthy, \cite{tian2021online} consider the zero-sum unknown model setting, and develop an algorithm that provides a $O(T^{2/3})$ regret guarantee when comparing to the minimax game value. \cite{tian2021online} also present a certain extension of their result to general-sum games, however their definition of regret in this case does not translate to the usual notion of regret even when all players adopt their algorithm.

Learning in \emph{general-sum} Markov games has been comparatively less explored. 
The work of \cite{liu2021sharp} presented a \emph{centralized} algorithm in the unknown model setup with optimal sample complexity guarantees in terms of the number of episodes, but exponential dependence on the number of agents. Following their work, several recent papers \citep{jin2021v, song2021can, mao2022provably} independently develop variants of V-learning, a decentralized algorithm for learning unknown general-sum Markov games. After $T$ episodes, their algorithms output a (non-Markov) $O(T^{-1/2})$-coarse correlated equilibrium, without dependence on the number of agents.
\cite{jin2021v} and \cite{song2021can} also present extensions for obtaining approximate (non-coarse) correlated equilibrium with similar guarantees. 
Later, \cite{mao2022improving} further propose simplifications to the V-learning algorithmic and analysis framework.
Notably though, the output of these algorithms is linear in the number of episodes (as it includes the history of all policies), and it is unclear what are the online guarantees of these methods.

The recent works of \cite{liu2022learning,zhan2022decentralized} explore the \emph{policy-revealing} setting, where agents share their policies after every episode. \cite{liu2022learning} give both positive and negative results on regret guarantees in this setup for zero-sum games, with the regret upper bounds depending on the cardinality of either the baseline or opponent policy classes. \cite{zhan2022decentralized} extend their work and present an algorithm for the policy-revealing setting with function approximation, which achieves no-regret in general-sum games in face of arbitrary opponents, as long as these reveal their policies at the end of each episode. Importantly, in both \cite{liu2022learning} and \cite{zhan2022decentralized} the computational complexity depends on the cardinality of the baseline policy class, and thus their algorithm is inefficient whenever the baseline policy class is the class of all Markov policies, as in our case.
Finally, \cite{zhang2022policy}
consider the full-information setting similar to ours (although, they do not develop extensions to the minimum reachability setup), and present algorithms that output $\tilde O(T^{-5/6})$ and $\tilde O(T^{-3/4})$ optimal policies after $T$ episodes, for respectively zero-sum and general-sum games; notably, however, it is unclear whether their algorithms provide regret guarantees.
 
\paragraph{Hardness results for Markov games.}

Learning in Markov games is considered a notoriously challenging problem, and several learning objectives have been shown in previous works to be either computationally or statistically hard. For instance, \cite{bai2020near} show that computing the best response policy in zero-sum Markov games against an adversarial opponent is at least as hard as learning parities with noise, a problem conjectured to be computationally hard. \cite{tian2021online} and \cite{liu2022learning} show a regret lower bound of $\Omega(\min\{\sqrt{2^H T}, T\})$ for zero-sum episodic Markov games with an unknown transition function, where the opponent is restricted to Markov policies. 
We note that these hardness results do not directly impact our goal of no-regret learning in general-sum Markov games, as they consider the setting of facing an arbitrary opponent which is only constrained to play Markov policies. By contrast, our main result shows that each player's individual regret is sublinear in $T$ as long as the other players' policies have a well-bounded second order path length, which is a property enforced by our choice of algorithm for all players.
Additionally, \cite{daskalakis2022complexity} show that the problem of computing a coarse correlated equilibrium comprised of stationary Markov policies in a general-sum infinite horizon Markov game is computationally hard. We consider a setting of regret minimization in layered episodic Markov games, and though our policies of interest are stationary, they do not translate into stationary policies in a corresponding infinite horizon Markov game. Hence, this lower bound is not applicable in the setting we consider here.

\paragraph{No-regret learning in games.}

Theoretically understanding no-regret dynamics in multiplayer games has been a topic of vast interest in recent years (e.g., \citealp{rakhlin2013optimization,syrgkanis2015fast,foster2016learning,chen2020hedging,daskalakis2021near,anagnostides2022uncoupled,piliouras2021optimal}). The main focus in most of these works is to analyze the performance of optimistic variants of online learning algorithms such as FTRL and OMD in multiplayer normal form games, and ultimately prove regret bounds which are vastly better than the naive $O(\sqrt T)$ guarantee achievable in adversarial environments. The state-of-the-art result in this setting was established by \cite{anagnostides2022uncoupled} who proposed an algorithm which guarantees $O(\log T)$ swap regret in general-sum games. Some of these results have been extended to more general classes of games such as extensive-form games~\citep{farina2022kernelized,anagnostides2022faster} and convex games \citep{farina2022near}. In this work we adopt some of the techniques presented by \citet{anagnostides2022uncoupled} in order to establish sublinear swap regret guarantees in general-sum Markov games.

\section{Preliminaries}

\subsection{Problem setup}
\paragraph{Markov games.}
An $m$-player general-sum finite horizon Markov game is defined by the tuple $\b{H,\S,\{\A_i\}_{i=1}^m,P,\{\ell^i\}_{i=1}^m}$. $H$ is the horizon; 
$\S$ is set of states of size $S$  partitioned as $\S = \bigcup_{h=1}^{H+1} \S_h$, where $\S_1 = \{s_1\}$ and $\S_{H+1}=\{s_{H+1}\}$;
$\A_i$ is the set of  actions of agent $i$ of size $A_i$, and the joint action space is denoted by $\A \coloneqq \bigtimes_{i=1}^m \A_i$. Further, $P$ is the transition kernel, where
 given the state at time $h$, $s\in\S_h$, and a joint action profile $\ba \in \A$, 
$P(\cdot \mid s, \ba) \in \Delta_{\S_{h+1}}$ is the probability distribution over the next state, where given some set $\C$, $\Delta_\C \eqq \cb{p : \C \to [0,1] \mid \sum_{x \in \C} p(x) = 1}$ denotes the probability simplex over $\C$. Finally, 
$\ell^i:\S\times \A \to [0,1]$ denotes the cost function of agent $i$. A policy for player $i$ is a function $\pi^i(\cdot \mid \cdot): \A_i \times \S \to [0,1]$, such that $\pi^i( \cdot \mid s) \in \Delta_{\A_i}$ for all $s\in \S$.
Given a policy profile $\bpi = (\pi^1,...,\pi^m)$, player $i\in[m]$, state $s \in \S_h$ and action $a \in \A_i$, we define the value function and the $Q$-function of agent $i$ by: 
\begin{align*}
    V^{i,\bpi}(s) = \E \sb{ \sum_{h' = h}^H \ell^i(s_{h'},\ba_{h'}) \mid \bpi, s_h = s}
    \quad ; \quad
    Q^{i,\bpi}(s,a) = \E \sb{ 
    \sum_{h' = h}^H \ell^i(s_{h'},\ba_{h'}) \mid a_h^i=a, s_h = s, \bpi}
    .
\end{align*}

\paragraph{Interaction protocol.}

The agents interact with the Markov game over the course of $T$ episodes.
At the beginning of each episode $t\in[T]$ every agent chooses a policy $\pi_t^i$.
Then, all agents start at the initial state $s_1$, and for each time step $h = 1,2,\ldots,H$, each player draws an action $a^i_h \sim \pi_t^i(\cdot \mid s_h)$ and the agents transition together to the next state $s_{h+1}\sim P(\cdot \mid s_h,\ba_h)$ where $\ba_h = (a^1_h,...,a^m_h)$.
At the end of the episode, agent $i$ incurs a loss given by $\sum_{h=1}^H \ell^i(s_h,\ba_h)$, and observes feedback that differs between two distinct settings we consider.
In the \emph{\textbf{full-information}} setup, agent $i\in[m]$ observes the exact state-action value functions; $Q^{i, \bpi_t}(s, a) , \;\; \forall s, a \in \S \times \A_i$.

\paragraph{Learning objective.}

Given an agent $i\in[m]$ and policy profile $\bpi = (\pi^1,...,\pi^m) = \pi^1 \odot \pi^2 \odot \ldots \odot \pi^m$, we will be interested in the policy profile excluding $i$ which we denote by $\bpi^{-i} \eqq \pi^1 \odot \ldots \pi^{i-1} \odot \pi^{i+1} \odot \ldots \odot \pi^m$. For policy profile $\bpi$ and a player $i$ policy $\pi \in \S \to \A_i$, we let $\pi \odot \bpi^{-i} = \pi^1 \odot \ldots \pi^{i-1} \odot \pi \odot \pi^{i+1} \odot \ldots \odot \pi^m$ denote the joint policy formed by replacing $\pi^{i}$ with $\pi$.
Given an episode $t$ and policy profile $\bpi_t = (\pi_t^1,...,\pi_t^m)$, we will be interested in the single agent MDP induced by $\bpi_t^{-i}$. This induced MDP is specified by
	$
		M_t^i \eqq (H, \S, \A_i, P_t^i, \l_t^i)
	$,
	where $\l_t^i(s, a) \eqq \E_{\ba\sim\bpi_t(\cdot \mid s)}\sb{\ell^i(s,\ba) \mid a^i = a}$ and 
	$P_t^i(\cdot \mid s,a) = \E_{\ba\sim\bpi_t(\cdot \mid s)} \sb{P(\cdot \mid s,\ba) \mid a^i = a}$ define agent $i$'s induced loss vector and transition kernel respectively.
	Furthermore, we denote the value and action-value functions of a policy $\pi \in \S \to \Delta_{\A_i}$  in this MDP by
\begin{align*}
    V_t^{i,\pi}(s) \eqq  V^{i,\pi \odot \bpi_t^{-i}}(s)
    \quad ; \quad 
    Q_t^{i,\pi}(s, a) \eqq Q^{i, \pi \odot \bpi_t^{-i}}(s,a)
    ,
\end{align*}
	 where $s \in \S_h$ and $a\in \A_i$. 
	 Given our definitions above, a standard argument shows that $V_t^{i,\pi}(s) 
    = \E \sb[b]{ \sum_{h' = h}^H \ell_t^i(s_{h'},a_{h'}) 
    	\mid P_t^i, \pi, s_h = s}$, and $Q_t^{i,\pi}(s, a)
    = \E \sb[b]{ \sum_{h' = h}^H \ell_t^i(s_{h'},a_{h'}) 
    	\mid P_t^i, \pi, s_h = s, a_h = a}$.
	 We note that we sometimes use the shorthand $V_t^i(\cdot)$ for $V_t^{i,\pi^i_t}(\cdot)$ and $Q_t^i(\cdot,\cdot)$ for $Q_t^{i,\pi^i_t}(\cdot,\cdot)$.
Given a jointly generated policy sequence $\{\bpi_1, \ldots, \bpi_T\}$,
our primary performance measure is the individual swap regret of each player $i$, defined as
\begin{align}
	\SwRe_T^i
	&= \max_{\phi_\star^i \in 
	\{\S \times \A_i \to \A_i\}} \cb[B]{
		\sum_{t=1}^T \b{V_t^{i, \pi^i_t}(s_1) - V_t^{i, \phi^i_\star(\pi^i_t)}(s_1)}
	},
\end{align}
where we slightly overload notation and define a policy swap function $\phi\colon \S \times \A_i \to \A_i$ applied to a policy $\pi \in \S \to \Delta_{\A_i}$ as follows; 
\begin{align*}
    \phi(\pi)(a \mid s) \eqq 
    \sum_{a':\phi(s, a')=a} \pi( a' \mid s).
\end{align*}
That is, the distribution $\phi(\pi)(\cdot \mid s)$ is formed by sampling $a \sim \pi(\cdot \mid s)$ and then replacing it with $\phi(s, a) \in \A_i$. Similarly, given an action swap function $\psi^{i} \colon \A_i \to \A_i$, we slightly overload notation when applying it to $x \in \Delta_{\A_i}$ by defining
    $
        \psi^{i}(x)(a) = 
        \sum_{a':\psi^{i}(a')=a} x( a')
    $.
We remark this notion of regret is strictly stronger (in the sense that it is always greater or equal) than the external regret, defined by;
\begin{align}
	\Re_T^i
	&\eqq \max_{\pi_\star^i \in \{\S \to \Delta_{\A_i}\}} \cb[B]{
		\sum_{t=1}^T \b{V_t^{i, \pi^i_t}(s_1) - V_t^{i, \pi^i_\star}(s_1)}
	}.
\end{align}
Finally, a joint policy distribution $\Pi$ is an $\eps$-approximate \emph{correlated equilibrium} if for any player $i$,
\begin{align}\label{eq:def_CE}
    \mathbb{E}_{\bpi \sim \Pi} \sb{\max_{\phi_i} \b{V^{i,\bpi}(s_1) - V^{i,(\phi_i(\pi^i) \odot \bpi^{-i})}(s_1)}} \leq \eps.
\end{align}
It is straightforward to show that if all players achieve swap regret of $O(\eps T)$ over $T$ episodes, then the distribution given by sampling $\bpi_t$ with $t \sim [T]$ uniformly constitutes an $\eps$-approximate correlated equilibrium~\citep{blum2007external}.

\paragraph{Additional notation.}

We denote the size of the largest action set as
$A \eqq \max_i A_i$.
	 In addition, we let $q_t^{i, \pi}$ denote the state-occupancy measure of policy $\pi$ in $M_t^i$;
\begin{align*}
	q_t^{i, \pi}(s) \eqq \Pr(s_h = s \mid P_t^i, \pi).
\end{align*}
Finally, for any pair of policies $\pi, \tilde \pi \in \S \to \Delta_{\A_i}$ of player $i$, we define
\begin{align*}
    \norms{\pi - \tilde \pi}_{\infty,1}
    \eqq 
    \max_{s\in \S} 
        \norms{\pi(\cdot \mid s) - \tilde \pi(\cdot \mid s)}_1,
\end{align*}
and for any $P, \tilde P \in \S \times \A_i \to \Delta_{\S}$,
\begin{align*}
    \norms{P - \tilde P}_{\infty,1}
    \eqq 
    \max_{s\in \S, a\in \A_i} 
        \norms{P(\cdot \mid s, a) - \tilde P(\cdot \mid s, a)}_1.
\end{align*}

\subsection{Optimistic online mirror descent}
Let $\X \subset \Delta_d$ be a convex subset of the $d$-dimensional simplex, and $\l_1, \ldots, \l_T\in [0, 1]^d$ be an online loss sequence.
Optimistic online mirror descent (OOMD) over $\X$ with convex regularizer $R\colon \X \to \R$ and learning rate $\eta > 0$ is defined as follows:
\begin{align*}
	\tilde x_0 &\gets \argmin_{x \in \X} R(x) ;
	\\
	t = 1, \ldots, T; \quad x_{t} &\gets \argmin_{x \in \X} \big\{
		\ab[s]{\tilde \l_t, x} + \frac1\eta D_R(x , \tilde x_{t-1})
	\big\},
	\\
	\tilde x_{t} &\gets \argmin_{x \in \X} \big\{
		\ab{\l_t, x} + \frac1\eta D_R(x , \tilde x_{t-1})
	\big\}.
\end{align*}
We instantiate OOMD with 1-step recency bias, meaning $\tilde \l_1 \eqq 0$ and $\tilde \l_t\eqq \l_{t-1}$ for $t\geq 2$.
The primal and dual local norms induced by the regularizer $R$ are denoted;
\begin{align*}
	\forall v\in \R^d, x\in \X, \qquad 
	\norm{v}_x =  \sqrt{v\T \nabla^2 R(x) v};
	\quad 
	\norm{v}_{*, x} =  \sqrt{v\T (\nabla^2 R(x))^{-1} v}.
\end{align*}
For the most part, we will employ the log-barrier regularization specified by
\begin{align}
    \forall x\in \X, \quad 
	R(x) = \sum_{a\in[d]} \log \frac{1}{x(a)}
	.
	\label{eq:log_barrier}
\end{align}
The Bregman divergence induced by the log-barrier is given by $D_R(x, y) = \sum_{a\in[d]} \log\frac{y(a)}{x(a)} + \frac{x(a) - y(a)}{y(a)}$, and the local norms by
\begin{align*}
	\forall v\in \R^d, x\in \X, \qquad 
	\norm{v}_x =  \sqrt{\sum_{a\in [d]} \frac{v(a)^2}{x(a)^2}};
	\quad 
	\norm{v}_{*, x} =  \sqrt{\sum_{a\in [d]} v(a)^2 x(a)^2}.
\end{align*}
Finally, throughout, we refer to the $\gamma$-\emph{truncated} simplex, defined by:
\begin{align}
	\Delta_d^\gamma 
		\eqq \cb{x \in \Delta_d \mid x(a) \geq \gamma, \quad \forall a \in [d]}.
	\label{eq:trunc_simplex}
\end{align}

\section{Algorithm and main result}
\label{sec:main}
In this section, we present our algorithm and outline the analysis establishing the regret bound.  
	We propose a policy optimization method with a carefully designed regret minimization algorithm employed in each state.
	Specifically, inspired by the work of \cite{anagnostides2022uncoupled}, we equip the swap regret algorithm of \cite{blum2007external} with a variant of optimistic online mirror descent over the truncated action simplex.
	This choice has two important properties; 
	First, it can be shown that online mirror descent (as well as its optimistic variant) with some tuning satisfies weighted regret bounds of the form that emerges from non-stationarity in MDP dynamics which directly depends on the path length of the joint policy sequence.
	Second, the (second-order) path length of the generated policy sequence is $O(\log T)$, a fact we establish by suitable modifications of the arguments presented in  \cite{anagnostides2022uncoupled}.
	
\begin{algorithm}[ht]
    \caption{Policy Optimization by Swap Regret Minimization}
    \label{alg:POSR}
	\begin{algorithmic}[1]
	    \STATE \textbf{input:}  
	        $H,\S,\calA_i,T$ agent index $i$, learning rate $\eta > 0$, regularizer $\reg(\cdot)$.
	   \STATE \textbf{initialization:} $\pi_{1}^{i}$ is the uniform policy. For every $s \in \S$ and every $a \in \calA_i$ initialize \\ $\Tilde{x}^{i,s,a}_{0}(\cdot) = \argmin_{x \in \domifulli} \reg(x)$.
	    \FOR{$t=1$ to $T$}
    	        \STATE Play policy $\pi_{t}^{i}$, 
    	        \STATE Observe an $\eps$-approximation of $Q^i_t(s,a)$ denoted by $\hat{Q}^i_t(s,a)$ for all $s,a$.
	            \STATE Incur the expected loss of the policy $\pi_t^i$ with respect to the losses $\ell_t^i$: $V_t^i(s_0)$.
                \STATE {\color{gray} \# Optimistic OMD step}
                \STATE For every $s,a$ perform an optimistic OMD update with the loss vector $g^{i,s,a}_t \coloneqq \pi^i_t(a \mid s) \hat{Q}^i_t(s, \cdot)$:
                \begin{align*}
                 \Tilde{x}^{i,s,a}_{t}(\cdot)
                    &=
                    \argmin_{x \in \domifulli} \cb{\eta \dotprod{x}{g^{i,s,a}_t} + D_{\reg} \b{x , \Tilde{x}^{i,s,a}_{t-1}(\cdot)}} \\ x^{i,s,a}_{t+1}(\cdot)
                    &=
                    \argmin_{x \in \domifulli} \cb{\eta \dotprod{x}{g^{i,s,a}_t} + D_{\reg}\b{x , \Tilde{x}^{i,s,a}_{t}(\cdot)}}
                \end{align*}
                \STATE {\color{gray} \# Policy update}
                \STATE For every state $s$ calculate $\pi^i_{t+1}(\cdot \mid s)$ - the stationary distribution corresponding to $\cb{x^{i,s,a}_{t+1}(\cdot)}_{a \in \calA_i}$ as follows: Let $B$ be the matrix whose rows are $\cb{x^{i,s,a}_{t+1}(\cdot)}_{a \in \calA_i}$, and let $\pi^i_{t+1}(\cdot \mid s) \in \domifulli$ be the distribution satisfying
                \begin{align*}
                    B \pi^i_{t+1}(\cdot \mid s) = \pi^i_{t+1}(\cdot \mid s).
                \end{align*}
        \ENDFOR
	\end{algorithmic}
\end{algorithm}

We remark that the policy update step (described in line 9 of \cref{alg:POSR}) can be performed in polynomial time, since it only requires solving a system of linear equations under linear inequality constraints.
We refer to \cref{alg:POSR}'s components in different levels by different names. Specifically, we refer to the components which perform the OOMD steps at a given state and action (see line 8 of \cref{alg:POSR}) as \emph{base algorithms}. On the level above the base algorithms, the components which perform the policy update from each state are referred to as \emph{state algorithms}.
The guarantee of \cref{alg:POSR} is provided in the statement of our theorem below. 

\begin{theorem}\label{thm:regret}
    Assume $T, H \geq 2$, and that all players adopt \cref{alg:POSR} with log-barrier regularization (\cref{eq:log_barrier}) and step size $\eta = \frac{1}{96 H^2 m \sqrt{S A}}$.
    Then, the swap regret of every player $i$ is bounded as
    \begin{align*}
        \SwRe_T^i
        &\leq 
        {10^4 H^{4} S A^{3} m^{2} 
        \log^2 (HAT) D_{\max}^i}\sqrt{T}
        + {600 m H \sqrt{S} A^{3/2} \log(AT)D_{\max}^i } \eps T
        \\
        &\quad 
        + 150 m H^2 S^{3/2} A^{7/2} 
            \log (HAT)
            + 2H^2 
        ,
    \end{align*} 
    where $D_{\max}^i \eqq \max_{s,a,t, a'}\frac{1}{\tilde x_t^{i, s, a}(a')}$.
    In particular, if $\eps=0$
    and $T \geq 16 m^4 H^{4} S^{2} A^{8}$,
    \textbf{and in addition} $\boldsymbol {D_{\max}^i = O(T^{\alpha})}$,
    we obtain:
    \begin{align*}
        \SwRe_T^i
        = \tilde O\b{
            m^2 S H^4 A^{3} T^{1/2+\alpha}
        }
        .
    \end{align*}
\end{theorem}
\cref{thm:regret} hinges on two separate avenues in the analysis, each of which corresponds to the two properties mentioned earlier.
The first, outlined in \cref{sec:po_weighted_regret} and \cref{sec:reg_pathlength}, establishes the individual regret may be bounded by the path length of the jointly generated policy sequence. 
The second avenue, presented in \cref{sec:pathlength} largely follows the techniques of \cite{anagnostides2022uncoupled}, and establishes the jointly generated policy sequence indeed has a well bounded path length.
The proof of \cref{thm:regret} proceeds by standard arguments building on the aforementioned parts of the analysis, and is deferred to \cref{sec:proof-thm-regret-known}. 
To conclude this section, we obtain as an immediate corollary that the players' joint policy sequence converges to an approximate correlated equilibrium (\cref{eq:def_CE}) of the Markov game.
\begin{corollary}[Convergence to a correlated equilibrium]
\label{cor:correlated-eq}
If all players adopt \cref{alg:POSR} with the parameters and conditions specified in \cref{thm:regret}, then for every player $i$ it holds that 
\begin{align*}
    \mathbb{E}_{t \sim [T]} \sb{\max_{\phi_i} \b{V^{i,\pi^i_t}_t(s_1) - V^{i,\phi_i(\pi^i_t)}_t(s_1)}} = \tilde O \b{m^2 S H^{4} A^{3} T^{-1/2+\alpha}},
\end{align*}
where $t \sim [T]$ denotes the uniform distribution over $[T] = \cb{1,2,\ldots,T}$.
\end{corollary}
The proof follows immediately from the observation that the left-hand-side is $\SwRe_T^i/T$, and applying the result of \cref{thm:regret}.

\subsection{Policy optimization in non-stationary MDPs and weighted regret}
\label{sec:po_weighted_regret}

The analysis of policy optimization algorithms is often built upon a fundamental regret decomposition known as the value difference lemma (see \cref{lem:val_diff}). 
In the regime of regret minimization with a stationary transition function, the value difference lemma leads to $S$ regret expressions, each weighted by a constant factor, and thus amenable to standard analysis \cite[e.g., ][]{shani2020optimistic,cai2020provably}.
By contrast, in the Markov game setup, the single agent induced MDP is essentially non-stationary, thus applying the value difference lemma results in a sum of \emph{weighted regret} expressions, with 
weights given by the (changing) occupancy measure of the benchmark policy, and in particular are not known to the learner. A natural complexity measure of weighted regret minimization is the total variation of weights --- roughly speaking, the larger the total variation the harder the problem becomes (for example, note that with zero total variation the problem reduces to standard regret minimization).

Thus, we are motivated to consider a weighted regret objective with unknown weights, which may be seen as a generalization to the previously studied notion of \emph{adaptive regret} \citep{hazan2009efficient}.
As it turns out, the widely used FTRL for example, might suffer linear regret even if the weights' total variation is as small as a constant (see the discussion in \cite{hazan2009efficient} and in \cref{sec:ftrl_nonstationary_lb}). On the other hand, we show that OMD-style algorithms are resilient to weighting with small variations, as long as the Bregman divergence between the benchmark policy and the iterates of the algorithm is bounded and the step size is chosen appropriately. 
\cref{lem:oomd_weighted} below establishes a weighted regret bound for optimistic OMD. 
The proof follows from arguments typical to OMD analyses, and is thus deferred to \cref{sec:deferred}.
    \begin{lemma}\label{lem:oomd_weighted}
    Assume we run OOMD on a sequence of losses $\{\ell_t\}_{t=1}^T$ with a regularizer $R \colon \X \to \R$ that is $1$-strongly convex w.r.t.~$\norm{\cdot}$. Then, for any $x^\star\in \X$ and any weight sequence $\{q_t\}_{t=1}^T$, it holds that
    \begin{align*}
        \sum_{t=1}^T q_t \ab{x_t - x^\star, \l_t}
        &\leq 
        \frac{q_1 D_R(x^\star, \tilde x_0)}{\eta}
        + \frac{1}{\eta}\sum_{t=1}^T (q_{t+1} - q_t)D_R(x^\star, \tilde x_t)
            + \frac{\eta}{2} \sum_{t=1}^T q_t \norms{\l_t - \tilde \l_t}_*^2
        .
    \end{align*}
\end{lemma}

\subsection{Bounding regret by path length}
\label{sec:reg_pathlength}
Our algorithm runs an instance of the meta algorithm of \cite{blum2007external} in each state with OOMD as a base algorithm. It is straightforward to show that the meta algorithm inherits the desired property of weighting-resilience (see \cref{lem:master_bases_weighted}). By relating the weights $q_t^{i,\star}(s)$ with the \textit{first-order} path length we are able to show the following regret bound.

\begin{theorem}\label{thm:master_regret_by_path}
    Suppose every player $i$ adopts \cref{alg:POSR} with log-barrier regularization (\cref{eq:log_barrier}), and that $\norms{\hat Q_t - Q_t}_\infty \leq \eps$ for all $t$.
    Then, assuming $T, H \geq 2$, the swap-regret of player $i$ is bounded as
    \begin{align*}
        \SwRe_T^i
        &\leq 
        \frac{S A_i^2 \log(A_i T)}{\eta}
        + \frac{ A_i^2 H^2 \log(A_i T)D_{\max}^i}{\eta  }
            \sum_{j=1}^m
            \sum_{t=1}^T \norms{
                \pi^j_{t+1} - \pi^j_{t}}_{\infty,1}
        \\
        &\quad + 4 \eta m S A_i H^4
        \sum_{t=1}^T
        \sum_{j=1}^m \norms{\pi^j_{t+1} - \pi^j_{t}}_{\infty,1}^2
        + 4 \eta  \eps^2 S T + \eps H T +  2 H (H + \eps)
        ,
    \end{align*}
    where $D_{\max}^i \eqq \max_{s,a,t, a'}\frac{1}{\tilde x_t^{i, s, a}(a')}$.
\end{theorem}
\begin{proof} 
    Let $\phi_\star^{i} \colon \S \times \A_i \to \A_i$ be any policy swap function, specifying $S$ action swap functions $\phi_\star^{i, s} \eqq \phi_\star^{i}(\cdot, s) \colon \A_i \to \A_i$.
    By \cref{lem:val_diff} (value difference), we have;
    \begin{align}
        \nonumber
        \sum_{t=1}^T V_t^{i, \pi^i_t}(s_1) - V_t^{i, \phi_\star(\pi^i_t)}(s_1)
        &= \sum_{s\in \S} \sum_{t=1}^T q_t^{i,\star}(s)\ab{ 
                Q_{t}^{i}(s, \cdot), 
                \pi_t^i(\cdot \mid s) 
                - \phi_\star^{i,s}(\pi_t^i(\cdot \mid s))
            }
        \\
        \nonumber
        &= \sum_{s\in\S} \sum_{t=1}^{T}q_{t}^{i,\star}(s) \ab{
        Q_{t}^{i}(s,\cdot) - \hat{Q}_t^i(s, \cdot),
        \pi_{t}^{i}(\cdot \mid s) - \phi_\star^{i,s}(\pi_t^i(\cdot \mid s))
        }
        \\
        \nonumber
        &\qquad + \sum_{s\in\S}\sum_{t=1}^{T}q_{t}^{i,\star}(s)\ab{
        \hat{Q}_t^i(s,\cdot) ,
        \pi_{t}^{i}(\cdot\mid s) - \phi_\star^{i,s}(\pi_t^i(\cdot \mid s))
        }
        \\
        & \leq  \eps H T + \sum_{s\in\S}\sum_{t=1}^{T}
        	q_{t}^{i,\star}(s)\ab{
        \hat{Q}_t^i(s,\cdot) ,
        \pi_{t}^{i}(\cdot\mid s) - \phi_\star^{i,s}(\pi_t^i(\cdot \mid s))
        }
        .
    \label{eq:proof-regret-by-path:value-diff}
    \end{align}
    Now, let $x_{\star, \gamma}^{i, s, a} \in \argmin_{x \in \Delta_{\A_i}^\gamma} \norms{x - \phi_\star^{i, s}(e_a)}_1$ for $\gamma=1/(A_i T)$, hence
    (see \cref{lem:logbarrier_simple_props})
    $\norms{x_{\star, \gamma}^{i, s, a} - \phi_\star^{i, s}(e_a)}_1 \leq 1/T$.
    By \cref{lem:master_bases_weighted},
    it now follows that for all $s$;
    \begin{align*}
        \sum_{s\in\S} &\sum_{t=1}^T q_{t}^{i,\star}(s) \ab{\hat Q_t^i(s, \cdot), 
        	\pi^i_t(\cdot \mid s) - \phi_\star^{i,s}(\pi_t^i(\cdot \mid s))
        	}
        \\
        &= 
        \sum_{s\in\S} \sum_{a\in \A_i} \sum_{t=1}^T q_{t}^{i,\star}(s) 
            \ab{g_t^{i,s,a}, x_t^{i, s, a} - \phi_\star^{i, s}(e_a)}
        \\
        &= \sum_{s\in\S} \sum_{a\in \A_i} \sum_{t=1}^T q_{t}^{i,\star}(s) 
            \ab{g_t^{i,s,a}, x_t^{i, s, a} - x_{\star, \gamma}^{i, s, a}}
            + \sum_{s\in\S} \sum_{a\in \A_i} \sum_{t=1}^T q_{t}^{i,\star}(s) 
            \ab{g_t^{i,s,a}, x_{\star, \gamma}^{i, s, a} - \phi_\star^{i, s}(e_a)}
        \\
        &\leq 
            \sum_{s\in\S} \sum_{a\in \A_i} \sum_{t=1}^T q_{t}^{i,\star}(s) 
            \ab{g_t^{i,s,a}, x_t^{i, s, a} 
                - x_{\star, \gamma}^{i, s, a}}
            + 2\gamma A_i H(H+\epsilon)T
        \\
        &= 
        	\sum_{s\in\S} \sum_{a\in \A_i} \sum_{t=1}^T q_{t}^{i,\star}(s) 
            \ab{g_t^{i,s,a}, x_t^{i, s, a} 
                - x_{\star, \gamma}^{i, s, a}}
            + 2 H (H + \eps)
        ,
    \end{align*}
    where in the second to last transition we use H\"older's inequality and that $\sum_s q_t^{i,\star}(s) = H $ for all $t$.
    By \cref{lem:oomd_weighted}, this can be further bounded by;
    \begin{align}
        &\leq H (H + \eps)
        + \frac{S A_i^2 \log ( A_i T)}{\eta}
        \nonumber \\
        &\quad  + \frac{1}{\eta} \sum_{s\in \S} \sum_{a\in \A_i} \sum_{t=1}^T
            (q_{t+1}^{i,\star}(s) - q_{t}^{i, \star}(s)) 
                D_R(x_{\star, \gamma}^{i, s, a}, \tilde x_t^{i, s, a})
        + \frac{\eta}{2} \sum_{s\in \S} \sum_{a\in \A_i} \sum_{t=1}^T \norms{g_{t+1}^{i, s, a} - g_{t}^{i, s, a}}_\infty^2,
        \label{eq: regrets-path main}
    \end{align}
    where we have used the fact that $\tilde x_0^{i, s, a} \in \argmin_x R(x) $ and $x_{\star, \gamma}^{i, s, a}\in \Delta_{\A_i}^\gamma$ implies $D_R(x_{\star, \gamma}^{i, s, a}, \tilde x_0^{i, s, a}) \leq A_i \log\frac1\gamma$.
    Proceeding, to bound the first series, note that for any $s,a,t$:
    \begin{align*}
        D_R(x_{\star, \gamma}^{i, s, a}, \tilde x_t^{i, s, a})
        &=
        \sum_{a'}
        \log\frac{\tilde x_t^{i, s, a}(a')}{x_{\star, \gamma}^{i, s, a}(a')}
        + 
        \frac{x_{\star, \gamma}^{i, s, a}(a')}{\tilde x_t^{i, s, a}(a')} - 1
        \\
        &\leq 
        A_i\log(A_i T)
        + \max_{a'}\frac{1}{\tilde x_t^{i, s, a}(a')}
        \\
        &\leq 
        A_i\log(A_i T)
        + \max_{s,a,t, a'}\frac{1}{\tilde x_t^{i, s, a}(a')}
        \\
        &= 
        A_i\log(A_i T)
        + D_{\max}^i,
    \end{align*}
    and apply \cref{lem:occupancy_path} which relates the total variation of the weights to the first-order path length;
    \begin{align*}
        &\frac{1}{\eta} \sum_{s\in \S} \sum_{a\in \A_i} \sum_{t=1}^T
            (q_{t+1}^{i,\star}(s) - q_{t}^{i, \star}(s)) 
                D_R(x_{\star, \gamma}^{i, s, a}, \tilde x_t^{i, s, a})
        \\
        &\leq \frac
            {A_i (A_i\log(A_i T) + D_{\max}^i)}
            {\eta} \sum_{t=1}^T
            \norms{q_{t+1}^{i,\star}
                - q_{t}^{i, \star}}_1
        \\
        &\leq \frac{A_i^2 H^2 \log(A_i T)D_{\max}^i}{\eta} \sum_{t=1}^T
            \sum_{j\in[m]}
                \norms{\pi^{j}_{t+1} 
            - \pi^{j}_{t}}_{\infty, 1}.
    \end{align*}
    
    Finally, we show in \cref{lem:g_variation_bound} that the second-order loss variation term (second series term in \cref{eq: regrets-path main}) is bounded by the second-order path length up to an additive lower order term. Formally,
    \begin{align*}
        \frac{\eta}{2} \sum_{s\in \S} \sum_{a\in \A_i} \sum_{t=1}^T \norms{g_{t+1}^{i, s, a} - g_{t}^{i, s, a}}_\infty^2
        &\leq
        \frac{ \eta }{2}
        \sum_{s\in \S} \sum_{a\in \A_i} \sum_{t=1}^T 
        \b{2H^2 \sum_{j=1}^m \norms{\pi^j_{t+1} - \pi^j_{t}}_{\infty,1} + 2\eps \pi_{t+1}^i(a \mid s)}^2
        \\
        &\leq
        4 \eta m S A_i H^4
        \sum_{t=1}^T
        \sum_{j=1}^m \norms{\pi^j_{t+1} - \pi^j_{t}}_{\infty,1}^2
        + 4 \eta  \eps^2 S T
        .
    \end{align*}
    Plugging the bounds from the last two displays into \cref{eq: regrets-path main}, the result follows.
\end{proof}

\subsection{Bounding the path length}
\label{sec:pathlength}

In this section, we provide a brief outline of the path length analysis, central to our regret bound. 
At a high level, the arguments closely follow those of \cite{anagnostides2022uncoupled}, with suitable adjustments made to accommodate the elements in which our setting differs.
Specifically, these include 
the use of mirror descent (rather than OFTRL) over the truncated simplex, the fact that a single iterate of each player is comprised of outputs from a collection of $S$ state algorithms, and
that each of these algorithms optimizes w.r.t.~approximate $Q$-functions rather than the true ones. 
Below, we state the main theorem and subsequently describe the analysis at a high level, with most of the technical details deferred to \cref{sec:pathlength_proofs}.
\begin{theorem} \label{thm:path length-apprx}
    If each player uses \cref{alg:POSR} with $\eta = \frac{1}{96 H^2 m\sqrt{SA}}$  then the following path length bound holds on the jointly generated policy sequence;
    \begin{align*}
        \sum_{t=1}^{T} \sum_{i=1}^{m} \norm{\pi_{t+1}^{i} - \pi_{t}^{i}}_{\infty,1}^{2} 
        & \leq 
        776 S A^{3} m \log (H A T) + \frac{4 \eps^{2}T}{mH^{4}}.
    \end{align*}
\end{theorem}
Similar to \cite{anagnostides2022uncoupled}, \cref{thm:path length-apprx} is derived by establishing that the swap regret (defined in formally in \cref{sec:pathlength_proofs}) of each state algorithm satisfies an RVU property (\textbf{R}egret bounded by \textbf{V}ariation in \textbf{U}tilities; originally defined in \cite{syrgkanis2015fast}) with suitable norms. 
	Following that, we use the non-negativity of the swap regret in order to arrive at the desired conclusion by a simple algebraic manipulation.
In order to establish the aforementioned RVU property of the state algorithms, the first step consists of showing the base algorithms satisfy an RVU property with the local norms induced by the log-barrier.
\begin{lemma}[RVU property of OOMD with log-barrier and local norms]
\label{lem:rvu-base}
Let $[d]$ represent an action set with $d$ actions, and consider an online loss sequence $g_1,\ldots,g_T \in [0,H]^d$.
Assume we run OOMD over $\X = \Delta_d$ with log-barrier regularization and learning rate $\eta \leq \frac{1}{64 H}$. 
Then, for all $x^\star \in \X$ the following regret bound holds:
\begin{align*}
    \sum_{t=1}^T \dotprod{x_t - x^\star}{g_t}
    &\leq
    \frac{d \log(dHT)}{\eta}
    +
    4 \eta \sum_{t=1}^T \norm{g_t - g_{t-1}}^2_{*,x_t}
    -
    \frac{1}{576 \eta} \sum_{t=1}^T \norm{x_t - x_{t-1}}^2_{x_t}
    +1,
\end{align*}
where we define $p_0 = \argmin_{x \in \X} \reg(x)$.
\end{lemma}
Using \cref{lem:rvu-base} and the swap regret guarantee of \cite{blum2007external}, we arrive at an RVU-like property for each state algorithm's swap regret, one that involves the local norms of base algorithms;
\begin{align*}
	\SwRe_T^{i, s} 
	\leq 
\frac{| \calA_{i}|^{2} \log (|\calA_i| H  T)}{\eta} 
        + 4 \eta \sum_{t=1}^{T} \sum_{a \in \A_{i}} \norm{g_{t}^{i,s,a} - g_{t-1}^{i,s,a}}_{*,x_{t}^{i,s,a}}^{2} 
        - \frac{1}{576 \eta} \sum_{t=1}^{T} \sum_{a \in \A_{i}} \norm{x_{t}^{i,s,a} - x_{t+1}^{i,s,a}}_{x_{t}^{i,s,a}}^{2}
	.
\end{align*}
The negative term in the right-hand side can be converted to the state algorithm's path length in $L_1$-norm using arguments similar to those of \cite{anagnostides2022uncoupled} (see \cref{lemma:master-bases relation}).
Each base algorithms' path length of the loss vectors may also be related to its second-order policy path length, by invoking \cref{lem:g_variation_bound}.
This, combined with some standard arguments, leads to a swap-regret RVU property of each state algorithm
 w.r.t.~the $\norm{\cdot}_{\infty, 1}$ and $\norm{\cdot}_1$ norms.
\begin{theorem}
\label{thm:apprx_state_master_swap_rvu}
    If all players play according to \cref{alg:POSR} with $ \eta \leq \frac{1}{128H}$,
    then for any $i \in[m]$, $s \in\S$, we have
    \begin{align*}
        \SwRe_T^{i, s}
        &\leq 
        \frac{| \calA_{i}|^{2} \log ( |\calA_{i}| H T)}{\eta} 
        + 36 \eta \eps^{2}T 
        + 4H^{4}m \eta \sum_{t=1}^T \sum_{j=1}^{m} \norm{\pi_{t+1}^{j} - \pi_{t}^{j}}_{\infty,1}^{2} 
        \\
        &\quad - \frac{1}{576 \eta|\A_{i}|} \sum_{t=1}^{T} \norm{\pi_{t}^{i}(\cdot \mid s) - \pi_{t + 1}^{i}(\cdot \mid s)}_{1}^{2} + |\calA_{i}|.
    \end{align*}
\end{theorem}
At this point, \cref{thm:path length-apprx} follows easily from \cref{thm:apprx_state_master_swap_rvu}, by summing the swap regret upper bounds over all states and players, and rearranging.

\subsection[Proof of main theorem]{Proof of \cref{thm:regret}}
\label{sec:proof-thm-regret-known}
Having established that the regret of \cref{alg:POSR} may be bounded by the path length of the generated policy sequence (\cref{thm:master_regret_by_path}), and that the second order path length is well bounded (\cref{thm:path length-apprx}), the proof of our main theorem combines both results using relatively standard arguments. 
Notably, \cref{thm:master_regret_by_path} bounds the regret by the sum of both the first and second order path lengths while \cref{thm:path length-apprx} provides only a second order bound. Thus, a $\sqrt {mT}$ factor is ultimately incurred in the final bound.

\begin{proof}
    By \cref{thm:path length-apprx}, we have that,
    \begin{align*}
        \sum_{t=1}^{T} \sum_{i=1}^{m} \norm{\pi_{t+1}^{i} - \pi_{t}^{i}}_{\infty,1}
        &\leq
        \sqrt{m T
            \sum_{t=1}^{T} \sum_{i=1}^{m} 
            \norm{\pi_{t+1}^{i} - \pi_{t}^{i}}_{\infty,1}^2}
        \\
        &\leq 
        \sqrt{ 766 m^2 T S A^{3} \log (HAT) } 
            + \sqrt{\frac{4 \eps^{2}T^2}{H^{4}}}
        \\
        &\leq   30 m \sqrt S A^{3/2} 
            \sqrt { T \log (HAT) } 
            + \frac{2 \eps T}{H^2}.
    \end{align*}
    Thus, by \cref{thm:master_regret_by_path}; 
    \begin{align*}
        &\SwRe_T^i
        \\
        &\leq 
        \frac{S A_i^2 \log(A_i T)}{\eta}
        + \frac{ A_i^2 H^2 \log (A_i T)D_{\max}^i}{\eta}
            \sum_{j=1}^m
            \sum_{t=1}^T \norms{
                \pi^j_{t+1} - \pi^j_{t}}_{\infty,1}
        \\
        &\quad + 4 \eta m S A_i H^4 \sum_{j=1}^m 
            \sum_{t=1}^T 
             \norms{\pi^j_{t+1} - \pi^j_{t}}_{\infty,1}^2
        + 4 \eta  \eps^2 S T + \eps H T + 2H (H+\eps)
        \\
        &\leq 
        \frac{S A^2 \log(A T)}{\eta}
        + \frac{30 H^{2} A^{\frac72} m \sqrt S
            \log^2 (HAT) D_{\max}^i}{\eta }\sqrt{T} 
        + \frac{2 A^2 \log (AT) D_{\max}^i \eps T}{\eta}
        \\
        &\quad + 16 \eta S A \eps^2 T
        + 3200 \eta H^{4} S^{2} A^{4} m^{2} \log(HAT)
        \\
        &\quad + 4 \eta  \eps^2 S T + \eps H T + 2H (H+\eps)
        \\
        &\leq 
        { 10^4 H^{4} S A^{3} m^{2} 
        \log^2 (HAT) D_{\max}^i}\sqrt{T}
        + {600 m H \sqrt{S} A^{3/2} \log(AT)D_{\max}^i } \eps T
        \\
        &\quad 
        + 150 m H^2 S^{3/2} A^{7/2} 
            \log (HAT)
            + 2H^2 
        ,
    \end{align*}
    where the last transition follows from our choice of $\eta = \frac{1}{96 H^2 m\sqrt{SA}}$, and completes the proof.
\end{proof}

\subsection*{Acknowledgements}
This work was supported by the European Research Council (ERC) under the European Union’s Horizon 2020 research and innovation program (grant agreements No.~882396; 101078075), by the Israel Science Foundation (grants number 993/17, 2549/19, 3174/23, 1357/24), by the Len Blavatnik and the Blavatnik Family foundation, by the Yandex Initiative in Machine Learning at Tel Aviv University, 
by a grant from the Tel Aviv University Center for AI and Data Science (TAD). 

We would like to thank Khashayar Gatmiry and Noah Golowich for kindly bringing to our attention a mistake in the analysis in a previous version of the paper.

\bibliographystyle{abbrvnat}
\bibliography{main}

\newpage 
\appendix

\section{Deferred Proofs}
\label{sec:deferred}
\label{sec:proof:thm:regret}

\begin{lemma}
    \label{lem:g_variation_bound}
    Assuming $H\geq 2$ and $\norms{\hat Q_t - Q_t}_\infty \leq \eps$, it holds that for all $i\in[m],s\in \S,a\in \A_i$;
    \begin{align*}
        \norms{g_{t+1}^{i, s, a} - g_{t}^{i, s, a}}_\infty
        \leq 2H^2 \sum_{j=1}^m \norms{\pi^j_{t+1} - \pi^j_{t}}_{\infty,1} + 2 \eps \pi_{t+1}^{i}(a \mid s).
    \end{align*}
\end{lemma}
\begin{proof}[of \cref{lem:g_variation_bound}]
    First, note that,
    \begin{align*}
        \norm{g_{t+1}^{i,s,a}-g_{t}^{i,s,a}}_{\infty} 
        & = 
        \norm{\pi_{t+1}^{i}(a \mid s) \hat{Q}_{t+1}^{i}(s, \cdot) - \pi_{t}^{i}(a \mid s) \hat{Q}_{t}^{i}(s, \cdot)}_{\infty} 
        \\
        & \leq 
        \pi_{t+1}^{i}(a \mid s) \norm{\hat{Q}_{t+1}^{i}(s, \cdot)- \hat{Q}_{t}^{i}(s, \cdot)}_{\infty}
        + \hat{Q}_{t}^{i}(s, \cdot)|\pi_{t+1}^{i}(a \mid s) - \pi_{t}^{i}(a \mid s)| 
        \\
        & \leq
        \pi_{t+1}^{i}(a \mid s) \norm{Q_{t+1}^{i}(s, \cdot) - Q_{t}^{i}(s, \cdot)}
        + H|\pi_{t+1}^{i}(a \mid s) - \pi_{t}^{i}(a \mid s)| 
        \\
        & \quad 
        + \pi_{t+1}^{i}(a \mid s) \norm{\hat{Q}_{t+1}^{i}(s, \cdot) - Q_{t+1}^{i}(s, \cdot)}_{\infty}
        + \pi_{t+1}^{i}(a \mid s) \norm{\hat{Q}_{t}^{i}(s, \cdot) - Q_{t}^{i}(s, \cdot)}_{\infty} 
        \\
        & \leq 
        \pi_{t+1}^{i}(a \mid s) \norm{Q_{t+1}^{i}(s, \cdot) - Q_{t}^{i}(s, \cdot)}_{\infty}
        +H| \pi_{t+1}^{i}(a \mid s) - \pi_{t}^{i}(a \mid s)|
        + 2 \eps \pi_{t+1}^{i}(a \mid s)
    \end{align*}
    To bound the difference in $Q$-values, observe that for any $a'\in \A_i$, by \cref{lem:Q_diff} and \cref{lem:mdp_variation};
    \begin{align*}
        Q^i_{t+1}(s, a') - &Q^i_t(s, a')
            \\
            &= Q^{i, \pi_{t+1}^i}(s, a'; M_{t+1}^i) - Q^{i, \pi_t^i}(s, a'; M_t^i)
            \\
            &\leq H^2 \norm{\pi^i_{t+1} - \pi^i_t}_{\infty, 1}
            + (H^2 +1) \normb{P_{M_{t+1}^i} - P_{M_{t}^i}}_{\infty, 1}
            + (H+1)\normb{\l_{M_{t+1}^i}
                - \l_{M_{t}^i} }_\infty
            \\
            &\leq H^2 \norm{\pi^i_{t+1} - \pi^i_t}_{\infty, 1}
            + (H^2 +1) \sum_{j\neq i}
                \norms{\pi^{j}_{t+1} 
                - \pi^{j}_{t}}_{\infty, 1}
            + (H+1)\normb{\l_{M_{t+1}^i}
                - \l_{M_{t}^i} }_\infty
            \\
            &\leq H^2 \norm{\pi^i_{t+1} - \pi^i_t}_{\infty, 1}
            + (H^2 + H + 2) \sum_{j\neq i}
                \norms{\pi^{j}_{t+1} 
                - \pi^{j}_{t}}_{\infty, 1}
            .
    \end{align*}
    Thus, by our assumption that $H \geq 2$, we have
    \begin{align}
    \nonumber
        \norms{g_{t+1}^{i, s, a} - g_{t}^{i, s, a}}_\infty
        &\leq  (H^2 + H + 2) \sum_{j\neq i} \norms{\pi^j_{t+1} - \pi^j_{t}}_{\infty,1}
            + (H^2 + H) \norms{\pi^i_{t+1} - \pi^i_{t}}_{\infty,1} 
            + 2 \eps \pi_{t+1}^{i}(a \mid s)
        \\ \nonumber
        &\leq 2H^2 \sum_{j=1}^m \norms{\pi^j_{t+1} - \pi^j_{t}}_{\infty,1}+ 2 \eps \pi_{t+1}^{i}(a \mid s).
    \end{align}
\end{proof}

\begin{proof}[of \cref{lem:oomd_weighted}]
    Following similar arguments made in \cite{rakhlin2013optimization, syrgkanis2015fast}, we have
\begin{align*}
    \ab{\l_t, x_t - x^\star} 
    = \ab{\l_t, \tilde x_t - x^\star} 
    + \ab[s]{\tilde \l_t, x_t - \tilde x_t} 
    + \ab[s]{\l_t - \tilde \l_t, x_t - \tilde x_t},
\end{align*}
and, from optimality conditions of the first and second optimization steps, respectively;
\begin{align*}
    \ab[s]{\tilde \l_t, x_t - \tilde x_t} 
    &\leq \frac{1}{\eta}\b{
        D_R(\tilde x_t, \tilde x_{t-1} )
        - D_R(\tilde x_t, x_t) 
        - D_R(x_t, \tilde x_{t-1})
    }
    \\
    \ab{\l_t, \tilde x_t - x^\star} 
    &\leq \frac{1}{\eta}\b{
        D_R(x^\star, \tilde x_{t-1} )
        - D_R(x^\star, \tilde x_t) 
        - D_R(\tilde x_t, \tilde x_{t-1})
    }.
\end{align*}
This implies
\begin{align*}
    \ab{x_t - x^\star, \l_t}
    &\leq 
    \ab[s]{\l_t - \tilde \l_t, x_t - \tilde x_t}
    \\
    &\quad+ \frac{1}{\eta} \b{
        D_R(x^\star, \tilde x_{t-1})
        - D_R(x^\star, \tilde x_t)
        - D_R(\tilde x_t, x_t)
        - D_R(x_t, \tilde x_{t-1})
    }
    \\
    &\leq 
    \frac{\eta}{2}\norms{\l_t - \tilde \l_t}_{*}^2 
    + \frac{1}{2\eta}\norm{x_t - \tilde x_t}^2
    \\
    &\quad+ \frac{1}{\eta} \b{
        D_R(x^\star, \tilde x_{t-1})
        - D_R(x^\star, \tilde x_t)
        - \frac{1}{2}\norm{\tilde x_t - x_t}^2
        - D_R(x_t, \tilde x_{t-1})
    }
    \\
    &= 
    \frac{\eta}{2}\norms{\l_t - \tilde \l_t}_{*}^2 
    \\
    &\quad+ \frac{1}{\eta} \b{
        D_R(x^\star, \tilde x_{t-1})
        - D_R(x^\star, \tilde x_t)
        - D_R(x_t, \tilde x_{t-1})
    }
\end{align*}

Multiplying both sides by $q_t$ and summing over $t$, we obtain
\begin{align*}
    \sum_{t=1}^T q_t \ab{x_t - x^\star, \l_t}
    &\leq 
    \frac{q_1 D_R(x^\star, \tilde x_0)}{\eta}
    + \frac{1}{\eta}\sum_{t=1}^T (q_{t+1} - q_t)D_R(x^\star, \tilde x_t)
        + \frac{\eta}{2} \sum_{t=1}^T q_t \norms{\l_t - \tilde \l_t}_*^2
    .
\end{align*}
\end{proof}



\section[Path length analysis]{Path length analysis (proofs for \cref{sec:pathlength})}
\label{sec:pathlength_proofs}
In this section, we provide the full technical details of the analysis outlined in \cref{sec:pathlength}.
As mentioned, the arguments mostly follow those of \cite{anagnostides2022uncoupled}, who prove a similar result in the setting of general-sum games with access to the exact induced utility functions.
The analysis hinges on establishing an RVU property for the swap regret, and then exploiting the fact that it is non-negative.
The swap regret of a sequence of iterates $x_1, \ldots, x_T \in \Delta_d$ w.r.t.~a sequence of linear losses $g_1, \ldots, g_T \in [0,1]^d$, is defined by;
\begin{align}
    \SwRe_T = \max_{\phi \in \{[d] \to [d]\}} \cb{
        \sum_{t=1}^T \ab{g_t, x_t - \phi(x_t)}
    }.
    \label{eq:swap_regret}
\end{align}
Recall we slightly overload notation when applying a swap function $\phi\colon [d] \to [d]$ to $x\in \Delta_d$ as follows; 
\begin{align*}
    \phi(x)(a) = 
    \sum_{a':\phi(a')=a} x( a').
\end{align*}
That is, the distribution $\phi(x) \in \Delta_d$ is formed by sampling $a \sim x$ and then replacing it with $\phi(a) \in [d]$. 
We note that, since losses are linear, the external regret is obtained by mapping all actions to the one optimal in hindsight, and taking $\phi$ to be the identity mapping we see the swap regret is never negative, hence $\SwRe_T \geq \max(\Re_T, 0)$.
The next theorem is due to \cite{blum2007external} and provides for an essential building block in our analysis. We formulate the theorem in the context of our state algorithms, and incorporate the weighting of instantaneous regret, for which the original arguments go through with no modification. The proof below is provided for completeness.
\begin{theorem}[\cite{blum2007external}]
\label{lem:master_bases_weighted} Let $i\in [m]$, $s\in \S$. For any weight sequence $\{q_t\}_{t=1}^T, q_t\in[0, 1],$ and any action swap function $\phi_\star\colon \A_i \to \A_i$, we have that when player $i$ runs \cref{alg:POSR}, the following holds;
    \begin{align*}
        \sum_{t=1}^T q_t \ab{\hat Q_t^i(s, \cdot), 
        	\pi^i_t(\cdot \mid s) - \phi_\star(\pi^i_t(\cdot \mid s))}
        =
        \sum_{a\in \A_i} \sum_{t=1}^T q_t 
            \ab{g_t^{i,s,a}, x_t^{i, s, a} 
            - \phi_\star(e_a)}
        ,
    \end{align*}
    where $e_a(a') = \mathbb I\{a' = a\}$.
\end{theorem}
\begin{proof}[of \cref{lem:master_bases_weighted}]
    The policy played by the state algorithm of player $i$ at $s$ satisfies $
        \pi_t^i(a' \mid s) = \sum_{a \in \A_i} \pi_t^i(a \mid s) x^{i, s, a}(a')
    $ (see line 10 of \cref{alg:POSR}),
    thus;
    \begin{align*}
        \ab{\pi_t^i(\cdot \mid s), \hat{Q}_t^i(s, \cdot)}
        &= \sum_{a'\in \A_i}\pi_t^i(a' \mid s)\hat{Q}_t^i(s, a') 
        \\
        &= \sum_{a'\in \A_i}
                \sum_{a \in \A_i} \pi_t^i(a \mid s) x^{i, s, a}(a')
                    \hat{Q}_t^i(s, a') 
        \\
        &=  \sum_{a\in \A_i}
                \sum_{a' \in \A_i} x_t^{i, s, a}(a')
                    \pi_t^i(a \mid s) \hat{Q}_t^i(s, a') 
        \\
        &= \sum_{a\in \A_i} \ab{x_t^{i,s,a}, g_t^{i,s,a}},
    \end{align*}
    where we use the definition of $g_t^{i,s,a}$ in \cref{alg:POSR}.
    In addition, we have
    \begin{align*}
        \ab{\phi_\star(\pi_t^i(\cdot \mid s)), \hat{Q}_t^i(s, \cdot)}
        =
        \ab{\pi_t^i(\cdot \mid s), \hat{Q}_t^i(s, \phi_\star(\cdot))}
        &= \sum_{a\in \A_i} \pi_t^i(a \mid s) \hat{Q}_t^i(s, \phi_\star(a))
        \\
        &= \sum_{a\in \A_i}
        \pi_t^i(a \mid s) \ab{
            \phi_\star(e_a), \hat{Q}_t^i(s, \cdot)},
        \\
        &= \sum_{a\in \A_i} 
            \ab{\phi_\star(e_a), g_t^{i,s,a}}.
    \end{align*}
    Combining the above two displays, the result follows.
\end{proof}

Before proving \cref{thm:apprx_state_master_swap_rvu}
we give the proof of \cref{lem:rvu-base}, demonstrating the RVU property of OOMD with local norms. The analysis is similar to arguments made in previous works, such as \citet[Theorem 7]{wei2018more}.
\begin{proof}[of \cref{lem:rvu-base}]
Denote:
\begin{align*}
    \Tilde{x}_t
    &=
    \argmin_{x \in \X} \cb{ \eta \dotprod{g_t}{x} + D_{\reg} \b{x , \Tilde{x}_{t-1}}},
\end{align*}
where $\Tilde{x}_0 = \argmin_{x \in \X} \reg(x)$. Let $\tilde x^\star = (1-\frac{1}{GT}) x^\star + \frac{1}{dGT}\textbf{1} \in \Delta_d$. We bound the instantaneous regret as follows:
\begin{align*}
    \dotprod{x_t - x^\star}{g_t}
    &=
    \dotprod{\tilde{x}^{\star} - x^{\star}}{g_{t}}
    +
    \dotprod{\Tilde{x}_t - \tilde x^\star}{g_t}
    +
    \dotprod{x_t - \Tilde{x}_t}{g_{t-1}}
    +
    \dotprod{x_t - \Tilde{x}_t}{g_t - g_{t-1}}.
\end{align*}
The first term is bounded by,
\begin{align*}
    \dotprod{\tilde{x}^{\star} - x^{\star}}{g_{t}}
        \leq\dotprod{\frac{1}{d H T}\textbf{1}}{g_{t}}
            \leq\frac{1}{T}.
\end{align*}
Using first-order optimality conditions and the three point identity we have:
\begin{align*}
    \dotprod{\Tilde{x}_t - \tilde x^\star}{g_t}
    &\leq
    \frac{1}{\eta} \b{D_{\reg} \b{\tilde x^\star , \Tilde{x}_{t-1}} - D_{\reg} \b{\tilde x^\star , \Tilde{x}_t} - D_{\reg} \b{\Tilde{x}_t , \Tilde{x}_{t-1}}},
\end{align*}
and similarly 
\begin{align*}
    \dotprod{x_t - \Tilde{x}_t}{g_{t-1}}
    &\leq
    \frac{1}{\eta} \b{D_{\reg} \b{\Tilde{x}_t , \Tilde{x}_{t-1}} - D_{\reg} \b{\Tilde{x}_t , x_t} - D_{\reg} \b{\Tilde{x}_t, x_{t-1}}} \\
    &=
    \frac{1}{\eta} D_{\reg} \b{\Tilde{x}_t , \Tilde{x}_{t-1}} - \frac{1}{2 \eta}\norm{\Tilde{x}_t - x_t}^2_{y_t} - \frac{1}{2\eta} \norm{\Tilde{x}_t - x_{t-1}}^2_{z_t},
\end{align*}
where $y_t \in [\Tilde{x}_t, x_t]$ and $z_t \in [x_{t-1}, \Tilde{x}_t]$. Note that by \cref{lemma:bounded ratio} and the condition on $\eta$, it holds that $\frac{1}{y_t(a)} \geq \frac{1}{2x_t(a)}$ for all $a \in \calA$ and $\frac{1}{z_t(a)} \geq \frac{1}{6x_t(a)}$. Combining this with the above we obtain
\begin{align*}
    \dotprod{x_t - \Tilde{x}_t}{g_{t-1}}
    &\leq
    \frac{1}{\eta} D_{\reg} \b{\Tilde{x}_t , \Tilde{x}_{t-1}} - \frac{1}{8 \eta}\norm{\Tilde{x}_t - x_t}^2_{x_t} - \frac{1}{72\eta} \norm{\Tilde{x}_t - x_{t-1}}^2_{x_t}.
\end{align*}
Also, by H\"older's inequality and Young's inequality:
\begin{align*}
    \dotprod{x_t - \Tilde{x}_t}{g_t - g_{t-1}}
    &\leq
    \norm{g_t - g_{t-1}}_{*, x_t} \cdot \norm{x_t - \Tilde{x}_t}_{x_t} \\
    &\leq
    4 \eta \norm{g_t - g_{t-1}}_{*, x_t}^2 + \frac{1}{16 \eta} \norm{x_t - \Tilde{x}_t}_{x_t}^2.
\end{align*}
Summing the above over $t$ we obtain:
\begin{align*}
    \sum_{t=1}^T \dotprod{\Tilde{x}_t - \tilde x^\star}{g_t}
    &\leq
    \frac{D_{\reg}(x^\star , \Tilde{x}_0)}{\eta} + 4 \eta \sum_{t=1}^T \norm{g_t - g_{t-1}}_{*, x_t}^2 - \frac{1}{72 \eta} \sum_{t=1}^T \norm{\Tilde{x}_t - x_{t-1}}^2_{x_{t-1}}
    - \frac{1}{72 \eta} \sum_{t=1}^T \norm{x_t - \Tilde{x}_t}_{x_t}^2.
\end{align*}
First note that by definition of $\Tilde{x}_0$ and first-order optimality conditions, it holds that 
$
    D_\reg(\tilde x^\star , \Tilde{x}_0)
    \leq
    \reg(\tilde x^\star) - \reg(\Tilde{x}_0)
$. Note also that $\reg(\Tilde{x}_0) \geq 0$, and since $\tilde x^\star \geq \frac{1}{dGT}$, $\reg(\tilde x^\star) \leq d \log (dGT)$. Thus,
\begin{align*}
    D_\reg(x^\star , \Tilde{x}_0)
    \leq
    d \log (d G T).
\end{align*}
Now note:
\begin{align*}
    \norm{x_t - x_{t-1}}^2_{x_t}
    &\leq
    2 \norm{\Tilde{x}_t - x_{t-1}}^2_{x_{t}} + 2 \norm{x_t - \Tilde{x}_t}_{x_t}^2.
\end{align*}
Using \cref{lemma:bounded ratio} again we have $\norm{\Tilde{x}_t - x_{t-1}}^2_{x_t} \leq 4 \norm{\Tilde{x}_t - x_{t-1}}^2_{x_{t-1}}$ which gives
\begin{align*}
    \norm{x_t - x_{t-1}}^2_{x_t}
    &\leq
    8 \norm{\Tilde{x}_t - x_{t-1}}^2_{x_{t-1}} + 8 \norm{x_t - \Tilde{x}_t}_{x_t}^2.
\end{align*}
Combining all of the above we conclude the proof.
\end{proof}

Next, we provide the proof of the principal theorem establishing an RVU property of each state algorithm.

\begin{proof}[of \cref{thm:apprx_state_master_swap_rvu}]
    Using \cref{lem:master_bases_weighted,lem:rvu-base}, for any swap function $\phi$,
    \begin{align}
        \nonumber
        &\sum_{t=1}^{T}  \dotprod{\pi_{t}^{i}( \cdot \mid s) - \phi (\pi_{t}^{i}( \cdot \mid s))}{\hat{Q}_{t}^{i}(s, \cdot)} 
        \leq \sum_{a \in \A_{i}} \sum_{t=1}^{T} \dotprod{x_{t}^{i,s,a} - \phi( \boldsymbol{e}_{a})}{g_{t}^{i,s,a}} 
        \\
        \nonumber
        & \leq \frac{| \calA_{i}|^{2} \log (|\calA_{i}| H T)}{\eta} 
        + 4 \eta \sum_{t=1}^{T} \sum_{a \in \A_{i}} \norm{g_{t}^{i,s,a} - g_{t-1}^{i,s,a}}_{*,x_{t}^{i,s,a}}^{2} 
        - \frac{1}{576 \eta} \sum_{t=1}^{T} \sum_{a \in \A_{i}} \norm{x_{t}^{i,s,a} - x_{t+1}^{i,s,a}}_{x_{t}^{i,s,a}}^{2} + |\calA_{i}| 
        \\
        \nonumber
        & \leq \frac{| \calA_{i}|^{2} \log (|\calA_{i}| H T)}{\eta} 
        + 4 \eta \sum_{t=1}^{T} \sum_{a \in \A_{i}} \sum_{a' \in \A_{i}}(x_{t}^{i,s,a}(a'))^{2} \norm{g_{t}^{i,s,a} - g_{t-1}^{i,s,a}}_{\infty}^{2} 
        - \frac{1}{576 \eta} \sum_{t=1}^{T} \sum_{a \in \A_{i}} \norm{x_{t}^{i,s,a} - x_{t+1}^{i,s,a}}_{x_{t}^{i,s,a}}^{2} + |\calA_{i}| 
        \\
        & \leq \frac{| \calA_{i}|^{2} \log (|\calA_{i}| H T)}{\eta} 
        + 4 \eta \sum_{t=1}^{T} \sum_{a \in \A_{i}} \norm{g_{t}^{i,s,a} - g_{t-1}^{i,s,a}}_{\infty}^{2} 
        - \frac{1}{576 \eta} \sum_{t=1}^{T} \sum_{a \in \A_{i}} \norm{x_{t}^{i,s,a} - x_{t+1}^{i,s,a}}_{x_{t}^{i,s,a}}^{2} + |\calA_{i}| ,
        \label{eq:state-rvu-lemma first}
    \end{align}
    Now, from \cref{cor:OOMD-ratio}, 
    \begin{align*}
        & \av{1 - \frac{x_{t+1}^{i,s,a}(a')}{x_{t}^{i,s,a}(a')}} 
        \leq 32\eta \b{\norm{g_{t-1}^{i,s,a}}_{\infty} 
        + \norm{g_{t - 2}^{i,s,a}}_{\infty}} 
        \leq 32\eta H \b{\pi_{t-1}^{i}(a \mid s) 
        + \pi_{t - 2}^{i}(a \mid s)} 
        \\
        \Longrightarrow & \sum_{a \in \A_{i}} \max_{a'} \av{1 - \frac{x_{t+1}^{i,s,a}(a')}{x_{t}^{i,s,a}(a')}} \leq \frac{1}{2}.
    \end{align*}
    Therefore, using \cref{lemma:master-bases relation}, 
    \begin{align}
        \norm{\pi_{t}^{i}( \cdot \mid s) - \pi_{t+1}^{i}( \cdot \mid s)}_{1}^{2} 
        \leq
        | \A_{i}| \sum_{a \in \A_{i}} \norm{x_{t}^{i,s,a} - x_{t+1}^{i,s,a}}_{x_{t}^{i,s,a}}^{2}.
        \label{eq:state-rvu-lemma master-base-relation}
    \end{align}
    Combining \cref{eq:state-rvu-lemma first,eq:state-rvu-lemma master-base-relation,lem:g_variation_bound} completes the proof. 
\end{proof}

We conclude with the proof of \cref{thm:path length-apprx}, which follows easily from \cref{thm:apprx_state_master_swap_rvu}.
\begin{proof}[of \cref{thm:path length-apprx}]
    From \cref{thm:apprx_state_master_swap_rvu} and the fact that swap regret is non-negative, we have
    \begin{align*}
        0 &\leq \sum_{i=1}^m \sum_{s\in \S} \text{Swap} \Re_{T}^{i,s} 
        \\
        & \leq 
        \frac{m S A^{2} \log \frac{1}{\gamma}}{\eta} + 36 \eta \eps^{2}m S T + \sum_{t=1}^{T} \sum_{i=1}^{m} \b{4\eta H^{4}m^{2} S - \frac{1}{576 \eta A}} \norm{\pi_{t+1}^{i} - \pi_{t}^{i}}_{\infty,1}^{2},
    \end{align*}
    where we also used the fact that $\sum_{s\in\S}\norm{\pi_{t}^{i}(\cdot \mid s) - \pi_{t + 1}^{i}(\cdot \mid s)}_{1}^{2}
    \geq 
    \norm{\pi_{t+1}^{i} - \pi_{t}^{i}}_{\infty,1}^{2}$.
    Setting $\eta = \frac{1}{96 H^2 m\sqrt{SA}}$ and rearranging the terms completes the proof.
\end{proof}

\subsection[Log barrier Lemmas]{Log barrier lemmas for \cref{thm:apprx_state_master_swap_rvu}}
\label{sec:logbar_lemmas}

The following \cref{lemma:auxiliary ratio,lemma:bounded ratio} follows by the proof technique of \citet[Lemma 12]{jin2020simultaneously}; \citet[Lemma 9]{lee2020closer}.

\begin{lemma}
    \label{lemma:auxiliary ratio}
    Let $F\colon \Delta_d \to \R$ defined as $F(x) = \eta \langle p, \ell\rangle + D_R(x, x')$ for some $x' \in \Delta_d$, where $R$ is the log-barrier regularization. Suppose $\norm{\ell}_\infty \leq H$, and that $\eta \leq \frac{1}{8H}$.
    Then, for any $x''\in\Delta_d$ such that $\norm{ x'' - x'}_{x'} = 8 \eta \norm{\ell}_\infty$, we have
    \[
        F(x'') \geq F(x').
    \]
\end{lemma}

\begin{proof}
    By second-order Taylor expansion of $F$ around $x'$, there exist $\xi$
    is on the line segment between $x''$ and $x'$ such that, 
    \begin{align}
        \nonumber
        F(x'') & = F(x') + \nabla F(x')\T(x'' - x') + \frac{1}{2}(x'' - x')\T\nabla^{2}F(\xi)(x'' - x'),
        \\
        \nonumber
        & = F(x') + \eta \langle \ell,x'' - x'\rangle + \frac{1}{2}(x'' - x')\T\nabla^{2}F(\xi)(x'' - x')
        \\
        \nonumber
        & \geq F(x') - \eta \Vert \ell \Vert_{*,x'} \Vert x'' - x' \Vert_{x'} + \frac{1}{2}(x'' - x')\T\nabla^{2}F(\xi)(x'' - x') 
        \\
        \nonumber
        & = F(x') - 8 \eta^{2} \norm{\ell}_\infty  \Vert \ell \Vert_{*,x'}  + \frac{1}{2} \Vert x'' - x' \Vert_{\xi}^{2}
        \\
        & \geq F(x') - 8\eta^2 \norm{\ell}_\infty^2  + \frac{1}{2} \Vert x'' - x' \Vert_{\xi}^{2}
        \label{eq:auxiliary ratio lemma}
    \end{align}
    The second equality is since $\nabla F(x') = \eta \ell$, the first inequality
    is H\"older inequality, the last equality is since $ \norm{ x'' - x'}_{x'} = 8 \eta \norm{\ell}_\infty$
    and $\nabla^{2}F = \nabla^{2}R$, and the last inequality is since $\Vert \ell \Vert_{*,x'} \leq \norm{\ell}_\infty$. 
    Now, for all $a$, since $\eta \leq \frac{1}{8 H}$:
    \[
        \frac{|\xi(a) - x'(a)|}{x'(a)} 
        \leq
        \frac{|x''(a) - x'(a)|}{x'(a)} \leq \Vert x'' - x' \Vert_{x'} 
        \leq
        8 \eta H \leq 1.
    \]
    In particular, $\xi(a) \leq 2x'(a)$, which implies that 
    \begin{align*}
        \Vert x'' - x' \Vert_{\xi}^{2} & = \sum_{a} \b{ \frac{x''(a) - x'(a)}{\xi(a)}}^{2}
        \\
        & \geq \frac{1}{4}\sum_{a} \b{ \frac{x''(a) - x'(a)}{x'(a)}}^{2} 
        \\
        & = \frac{1}{4} \Vert x'' - x' \Vert_{x'}^{2} 
            = 16\eta^2 \norm{\ell}_\infty^2.  
    \end{align*}
    Plugging back in \cref{eq:auxiliary ratio lemma},
    \begin{align*}
        F(x'') 
        & \geq 
        F(x') - 8\eta^2 \norm{\ell}_\infty^2 + 16\eta^2 \norm{\ell}_\infty^2 \geq F(x'). 
    \end{align*}
\end{proof}

\begin{lemma}
    \label{lemma:bounded ratio}
    Let $x' \in \Delta_d$, $F\colon \Delta_d \to \R$ be defined as in \cref{lemma:auxiliary ratio} and let $x^{+} = \arg\min_{x}F(x)$.
    If $\eta \leq \frac{1}{8 H}$ then for any $a \in [d]$ it holds that 
    \[
        \left(1 - 8 \eta \norm{\ell}_\infty \right)x'(a) 
        \leq 
        x^{+}(a) 
        \leq 
        \left(1 + 8 \eta \norm{\ell}_\infty \right)x'(a).
    \]
\end{lemma}

\begin{proof}
    We first show that $ \Vert x^{+} - x' \Vert_{x'} \leq 8 \eta \norm{\ell}_\infty$.
    Assume otherwise: $ \Vert x^{+} - x' \Vert_{x'} > 8 \eta \norm{\ell}_\infty$. Then for some
    $\lambda\in(0,1)$ and $x'' = \lambda x^{+} + (1 - \lambda) x'$ we have,
    \[
        \Vert x'' - x' \Vert_{x'} = 8 \eta \norm{\ell}_\infty.
    \]
    From \cref{lemma:auxiliary ratio}, $F(x') \leq F(x'')$. Since $F$ is strongly convex and $x^{+}$
    is the (unique) minimizer of $F$, 
    \[
        F(x') 
        \leq F(x'') 
        \leq \lambda F(x^{+}) + (1 - \lambda) F(x') 
        < \lambda F(x') + (1 - \lambda) F(x') 
        = F(x'),
    \]
    which is a contradiction. Hence, $ \Vert x^{+} - x' \Vert_{x'} \leq 8 \eta \norm{\ell}_\infty$
    and for any $a$,
    \[
        \b{ \frac{x^{+}(a) - x'(a)}{x'(a)}}^{2} \leq \sum_{\tilde a} \b{ \frac{x^{+}(\tilde a) - x'(\tilde a)}{x'(\tilde a)}}^{2} \leq (8 \eta \norm{\ell}_\infty)^2.
    \]
    By rearranging the inequality above we obtain the statement of the lemma.
\end{proof}

\begin{corollary}
    \label{cor:OOMD-ratio}
    Assume that $\{x_t\}_{t=1}^T$ are iterates of OOMD with log-barrier and $\eta \leq \frac{1}{64 H}$, then for any $t$,
    \[
        1 - 32 \eta (\norm{\ell_{t-1}}_\infty + \norm{\ell_{t-2}}_\infty )
        \leq \frac{x_{t+1}(a)}{x_{t}(a)}
        \leq 
        1 + 32 \eta (\norm{\ell_{t-1}}_\infty + \norm{\ell_{t-2}}_\infty ).
    \]
\end{corollary}
\begin{proof}
    By \cref{lemma:bounded ratio} we have,
   \begin{align*}
        1 - 8 \eta \norm{\ell_{t-1}}_\infty
        & \leq 
        \frac{x_{t}(a)}{\tilde{x}_{t - 1}(a)} \leq 1 + 8 \eta \norm{\ell_{t-1}}_\infty
        \\
        1 - 8 \eta \norm{\ell_{t-1}}_\infty
        & \leq 
        \frac{\tilde{x}_{t - 1}(a)}{\tilde{x}_{t - 2}(a)} \leq 1 + 8 \eta \norm{\ell_{t-1}}_\infty
        \\
        (1 + 8 \eta \norm{\ell_{t-2}}_\infty)^{-1} 
        & \leq 
        \frac{\tilde{x}_{t - 2}(a)}{x_{t - 1}(a)} 
        \leq (1 - 8 \eta \norm{\ell_{t-2}}_\infty)^{-1}.
   \end{align*}
    Hence, 
   \[
       \frac{x_{t}(a)}{x_{t - 1}(a)}
       =
       \frac{x_{t}(a)}{\tilde{x}_{t - 1}(a)} \cdot \frac{\tilde{x}_{t - 1}(a)}{\tilde{x}_{t - 2}(a)} \cdot \frac{\tilde{x}_{t - 2}(a)}{x_{t - 1}(a)} 
       \leq 
       \frac{(1 + 8 \eta (\norm{\ell_{t-1}}_\infty + \norm{\ell_{t-2}}_\infty ))^{2}}{1 - 8 \eta (\norm{\ell_{t-1}}_\infty + \norm{\ell_{t-2}}_\infty )} 
       \leq 1 + 32 \eta (\norm{\ell_{t-1}}_\infty + \norm{\ell_{t-2}}_\infty )
   \]
     where the last is since $\frac{(1 + x)^{2}}{1 - x} \leq 1 + 4x$ for $x\in(0,1/5]$.
    In a similar way,
   \[
       \frac{x_{t}(a)}{x_{t - 1}(a)}
       =
       \frac{x_{t}(a)}{\tilde{x}_{t - 1}(a)} \cdot \frac{\tilde{x}_{t - 1}(a)}{\tilde{x}_{t - 2}(a)} \cdot \frac{\tilde{x}_{t - 2}(a)}{x_{t - 1}(a)} 
       \geq 
       \frac{(1 - 8 \eta (\norm{\ell_{t-1}}_\infty + \norm{\ell_{t-2}}_\infty ))^{2}}{1 + 8 \eta (\norm{\ell_{t-1}}_\infty + \norm{\ell_{t-2}}_\infty )}
       \geq 1 - 24 \eta (\norm{\ell_{t-1}}_\infty + \norm{\ell_{t-2}}_\infty )
   \]
    where the last is since $\frac{(1 - x)^{2}}{1 + x} \leq 1 - 3x$ for all $x>0$.
\end{proof}

The following lemma is a slight generalization of  \citep[Lemma 4.2]{anagnostides2022uncoupled}, which applies for any sequence of sufficiently stable base iterates (not necessarily OFTRL generated). 
\begin{lemma}[\citet{anagnostides2022uncoupled}] 
    \label{lemma:master-bases relation}
    Fix some vectors $x_{t,a} \in \Delta_d$ for $t \in [T]$ and
    $a \in [d]$. Let $M_{t} \in \mathbb {R}^{d\times d}$ a matrix
    who's rows are $x_{t,a}$ and let $x_{t} \in \Delta_d$
    be vectors such that $M_{t}\T x_{t} = x_{t}$. If $\sum_{a\in[d]} \max_{a'} \av{\frac{x_{t - 1,a}(a') - x_{t,a}(a')}{x_{t - 1,a}(a')}}  \leq \frac{1}{2}$, then
    \[
        \norm{x_{t} - x_{t - 1}} _{1}^{2} \leq 64A \sum_{a \in [d]} \norm{x_{t,a} - x_{t - 1,a}} _{x_{t - 1,a}}^{2}
    \]
    where the local norms here are those induced by the log-barrier regularizer (\cref{eq:log_barrier}).
\end{lemma}

\begin{proof}
    Let $\mathbb {T}_{a}$ be the set of all directed trees over $[d]$
    (i.e., each directed tree has no directed cycles, each node $a'\ne a$
    has exactly $1$ outgoing edge and $a$ has no outgoing edges). By
    Markov Chain Tree Theorem \citep{anantharam1989proof} $x_{t}(a) = \frac{w_{t}(a)}{W_{t}}$
    where 
    \[
        w_{t}(a) = \sum_{\mathcal{T} \in \mathbb {T}_{a}} \prod_{(u,v) \in E(\mathcal{T})}x_{t,u}(v)\qquad\text{and,} \qquad W_{t} = \sum_{a}w_{t}(a).
    \]
    Let $\mu_{t,a}: = \max_{a'} \av{1 - \frac{x_{t,a}(a')}{x_{t - 1,a}(a')}}$.
    In particular, $1 - \mu_{t,a} \leq \frac{x_{t,a}(a')}{x_{t - 1,a}(a')} \leq 1 + \mu_{t,a}$
    which implies
    \begin{align*}
        w_{t}(a) & \leq \sum_{\mathcal{T} \in \mathbb {T}_{a}} \prod_{(u,v) \in E(\mathcal{T})}x_{t,u}(v)
        \\
        & \leq \sum_{\mathcal{T} \in \mathbb {T}_{a}} \prod_{(u,v) \in E(\mathcal{T})}(1 + \mu_{t,u})x_{t - 1,u}(v)
        \\
        & \leq \prod_{a' \in [d]} (1 + \mu_{t,a'}) \sum_{\mathcal{T} \in \mathbb {T}_{a}} \prod_{(u,v) \in E(\mathcal{T})}x_{t - 1,u}(v) 
        \\
        & = \prod_{a' \in [d]} (1 + \mu_{t,a'})w_{t - 1}(a)
        \\
        & \leq \exp\b{ \sum_{a' \in [d]} \mu_{t,a'}} w_{t - 1}(a).
        \end{align*}
        This also implies that $W_{t} \leq \exp\b{ \sum_{a' \in [d]} \mu_{t,a'}} W_{t - 1}$.
        In a similar way,
        \begin{align*}
        w_{t}(a) & \geq \prod_{a' \in [d]} (1 - \mu_{t,a'})w_{t - 1}(a)
        \\
        & \geq \prod_{a' \in [d]} \exp\b{ - 2 \sum_{a' \in [d]} \mu_{t,a'}} w_{t - 1}(a) 
    \end{align*}
    where the last uses the fact that $1 - x \geq e^{ - 2x}$ for $x \in [0,1/2]$
    and that $\sum_{a' \in [d]} \mu_{t,a'} \leq 1/2$. Similarly,
    $W_{t} \geq \exp\b{ - 2 \sum_{a' \in [d]} \mu_{t,a'}} W_{t - 1}$. Combining
    the inequities above we get,
    \begin{align*}
        x_{t}(a) - x_{t - 1}(a) & = \frac{w_{t}(a)}{W_{t}} - x_{t - 1}(a)
        \\
        & \leq \frac{\exp\b{ \sum_{a' \in [d]} \mu_{t,a'}} w_{t - 1}(a)}{\exp\b{ - 2 \sum_{a' \in [d]} \mu_{t,a'}} W_{t - 1}} - x_{t - 1}(a) \\
        & \leq x_{t - 1}(a)\b{\exp\b{3 \sum_{a' \in [d]} \mu_{t,a'}} - 1}
        \\
        & \leq 8x_{t - 1}(a) \sum_{a' \in [d]} \mu_{t,a'},
    \end{align*}
    where the last holds since $e^{x} - 1 \leq \frac{8}{3}x$ for $x \in [0,2/3]$
    and $\sum_{a' \in [d]} \mu_{t,a'} \leq 1/2$. In similar
    way,
    \begin{align*}
        x_{t - 1}(a) - x_{t}(a) & \leq x_{t - 1}(a)\b{1 - \exp\b{ - 3 \sum_{a' \in [d]} \mu_{t,a'}} } \leq 3x_{t - 1}(a) \sum_{a' \in [d]} \mu_{t,a'}. 
    \end{align*}
    From the last to we have $\av{x_{t - 1}(a) - x_{t}(a)} \leq 8x_{t - 1}(a) \sum_{a' \in [d]} \mu_{t,a'}$
    and so,
    \begin{align*}
        \norm{x_{t - 1} - x_{t}} _{1}^{2} & \leq 64\b{ \sum_{a \in [d]} \mu_{t,a}}^{2}
        \\
        & \leq 64A \sum_{a \in [d]} \b{\mu_{t,a}}^{2}
        \\
        & \leq 64A \sum_{a \in [d]}  \sum_{a' \in [d]} \b{\frac{x_{t - 1,a}(a') - x_{t,a}(a')}{x_{t - 1,a}(a')}}^{2} 
        \\
        & = 64A \sum_{a \in [d]} \norm{x_{t,a} - x_{t - 1,a}} _{x_{t - 1,a}}^{2}.
    \end{align*}
\end{proof}

\section{Elementary MDP Lemmas}
In this section, we prove some basic lemmas relating variations in state visitation measures, losses and dynamics to the movement (changes in policies) of players.
Recall we let $\l_t^i, P_t^i, M_t^i$ denote respectively the loss, dynamics, and MDP tuple $M_t^i \eqq (H, \S, \A_i, P_t^i, \l_t^i)$ of the single agent induced MDP of player $i$ at round $t$. Further, for any (single agent) transition function $P$, policy $\pi\in \S \to \Delta_{\A_i}$ and state $s\in \S_h$, we denote
\begin{align*}
    q_{P}^\pi (s, a) 
        &\eqq \Pr(s_h=s, a_h=a \mid P, \pi, s_1),
    \\
    q_{P}^\pi (s) 
        &\eqq \Pr(s_h=s \mid P, \pi, s_1).
\end{align*}
When $P$ is clear from context we may omit the subscript and write $q^\pi$ for $q_P^\pi$.
Further, for any single agent MDP $M = (H, \S, \A_i, P, \l)$, we write 
$V(\cdot; M), Q(\cdot, \cdot; M)$ to denote respectively the state and state-action value functions of $M$. We may omit $M$ and write $V(\cdot), Q(\cdot, \cdot)$ when $M$ is clear from context. For $s\in \S$, we let $h(s) \eqq h \text{ s.t. } s\in \S_h$. With this notation in place, we have for a policy $\pi\in \S \to \Delta_{\A_i}$:
\begin{align*}
   V^\pi(s; M) 
   &\eqq \E \sb[b]{ \sum_{h = h(s)}^H \ell(s_{h},a_{h}) 
    	\mid P, \pi, s_{h(s)} = s},
    \\
    Q^\pi(s, a; M)
    &\eqq 
    \E \sb[b]{ \sum_{h = h(s)}^H \ell(s_{h},a_{h}) 
    	\mid P, \pi, s_{h(s)} = s, a_{h(s)} = a}.
\end{align*}
We begin with value difference lemmas which are typical in single agent MDP analyses.
The proofs below are provided for completeness; see also \cite{shani2020optimistic, cai2020provably} for similar arguments.
\begin{lemma}[value-difference]\label{lem:val_diff}
	The following holds.
	\begin{enumerate}
        \item For any MDP $M = (\S, \A_i, H, P, \l)$, and pair of policies $\pi, \tilde \pi \in \S \to \Delta_{\A_i}$, we have
        \begin{align*}
            V^\pi(s_1) - V^{\tilde \pi}(s_1)
            &= \E \sb{ 
                \sum_{h=1}^H \ab{
                    Q^\pi(s_h, \cdot), 
                    \pi(\cdot \mid s_h) - \tilde \pi(\cdot \mid s_h)
                } \mid \tilde \pi
            }
            \\
            &= \sum_{s\in \S} q^{\tilde \pi }(s)\ab{ 
                    Q^\pi(s, \cdot), 
                    \pi(\cdot \mid s) - \tilde \pi(\cdot \mid s)
                }
            \\
            &\leq H^2 \norm{\pi - \tilde \pi}_{\infty, 1}.
        \end{align*}
        \item For any two MDPs $M = (H, S, \A_i, P, \l), \tilde M = (H, S, \A_i, \tilde P, \tilde \l)$, 
        $V^\pi(\cdot) \eqq V^\pi(\cdot; M), 
            \tilde V^\pi(\cdot) \eqq V^\pi(\cdot; \tilde M)$,
        and policy $\pi\in \S \to \Delta_{\A_i}$, we have
        \begin{align*}
            V^\pi(s_1) - \tilde V^\pi(s_1)
            &= \E_{\tilde P, \pi} \sb{
                \sum_{h=1}^H \l(s_h, a_h) - \tilde \l(s_h, a_h)
                + \sum_{s'\in \S_{h+1}}
                 (P(s'|s_h, a_h) - \tilde P(s'|s_h, a_h))V^\pi (s')
            } 
            \\
            &\leq H \norm{\l - \tilde \l}_\infty 
            + H^2\norm{P - \tilde P}_{\infty, 1}
            \end{align*}
	\end{enumerate}
\end{lemma}
\begin{proof}
	For (1), observe that for $s\in \S_l$;
	\begin{align*}
		V^\pi(s) - V^{\tilde \pi}(s)
            &= \ab{Q^\pi(s, \cdot), \pi(\cdot |s) - \tilde \pi(\cdot |s)}	
            		+ \ab{Q^\pi(s, \cdot) - Q^{\tilde \pi}(s, \cdot), \tilde \pi(\cdot |s)}	
            	\\
            	&= \ab{Q^\pi(s, \cdot), \pi(\cdot |s) - \tilde \pi(\cdot |s)}	
            		+ \E_{a'\sim \tilde \pi(\cdot|s)} \sb{
            			\E_{s' \sim P(\cdot|s, a')} \sb{ 
	            			V^\pi(s') - V^{\tilde \pi}(s')
	            		}
            		} 
            	\\
            	&= \ab{Q^\pi(s, \cdot), \pi(\cdot |s) - \tilde \pi(\cdot |s)}	
            		+ \E\sb{
	            			V^\pi(s_{l+1}) - V^{\tilde \pi}(s_{l+1})
            		\mid \tilde \pi, s_l=s}.
	\end{align*}
	Applying the relation recursively, we obtain for $l=1$;
	\begin{align*}
		V^\pi(s_1) - V^{\tilde \pi}(s_1)
            &= \E\sb{ \sum_{h=1}^H
	            		\ab{Q^\pi(s_h, \cdot), \pi(\cdot |s_h) - \tilde \pi(\cdot |s_h)}	
            		\mid \tilde \pi, s_1}
            	\\
            	&= \sum_{h=l}^H \sum_{s\in \S_h} \Pr(s_h = s \mid, s_1, \tilde \pi)
	            		\ab{Q^\pi(s, \cdot), \pi(\cdot |s) - \tilde \pi(\cdot |s)}	
            	\\
            	&= \sum_{s\in \S} \Pr(s_h = s \mid, s_1, \tilde \pi)
	            		\ab{Q^\pi(s, \cdot), \pi(\cdot |s) - \tilde \pi(\cdot |s)}	
            	\\
            	&= \sum_{s\in \S} q^{\tilde \pi}(s)
	            		\ab{Q^\pi(s, \cdot), \pi(\cdot |s) - \tilde \pi(\cdot |s)}	
	        \\
	        &\leq H \sum_{s\in \S} q^{\tilde \pi}(s)
	            	 	\norm{\pi(\cdot |s) - \tilde \pi(\cdot |s)}_1
	        	\\
	        &\leq H \sum_{s\in \S} q^{\tilde \pi}(s)
	            		 \norm{\pi - \tilde \pi}_{\infty,1}
	        \leq H^2 \norm{\pi - \tilde \pi}_{\infty,1},
	\end{align*}
	which completes the proof of (1).
	For (2), let $s\in \S_l$ and observe;
	\begin{align*}
		V^\pi(s) - \tilde V^{\pi}(s)
            &= \E_{a\sim \pi(\cdot|s)} \sb{
            		\l(s, a) - \tilde \l(s, a)
            		+ \sum_{s'\in \S_{l+1}} P(s'|s, a)V^\pi(s')
            		- \tilde P(s'|s, a)\tilde V^\pi(s')
            }
    \end{align*}
    Further, we have
    \begin{align*}
    	\sum_{s'\in \S_{l+1}} & P(s'|s, a)V^\pi(s')
            		- \tilde P(s'|s, a)\tilde V^\pi(s') 
        \\
           	&= \sum_{s'\in \S_{l+1}} 
           		\b{P(s'|s, a) - \tilde P(s'|s, a)}V^\pi(s')
            		+ \tilde P(s'|s, a) \b{ V^\pi(s') - \tilde V^\pi(s')},
    \end{align*}
    and combining this with the previous equation we get
    \begin{align*}
		V^\pi(s) - \tilde V^{\pi}(s)
            &= \E_{a\sim \pi(\cdot|s)} \sb{
            		\l(s, a) - \tilde \l(s, a)
            		+ \sum_{s'\in \S_{l+1}} 
           		\b{P(s'|s, a) - \tilde P(s'|s, a)}V^\pi(s')
            }
            \\
            	&\quad + \E \sb{V^\pi(s_{l+1}) - \tilde V^\pi(s_{l+1}) \mid \tilde P, \pi, s_l=s}.
    \end{align*}
    Applying this recursivly with $l=1$, the first part of (2) follows.
    For the second part;
    \begin{align*}
    	\E_{\tilde P, \pi} &\sb{
                \sum_{h=1}^H \l(s_h, a_h) - \tilde \l(s_h, a_h)
                + \sum_{s'\in \S_{h+1}}
                 (P(s'|s_h, a_h) - \tilde P(s'|s_h, a_h))V^\pi (s')
            } 
         \\
         &\leq 
         \E_{\tilde P, \pi} \sb{
         	\sum_{h=1}^H \norm{\l - \tilde \l}_\infty
         	+ H \norm{P(\cdot|s_h, a_h) - \tilde P(\cdot|s_h, a_h)}_1
         }
         \\
         &\leq      
         \E_{\tilde P, \pi} \sb{
	        H \norm{\l - \tilde \l}_\infty + 
          	H \sum_{h=1}^H \norm{P - \tilde P}_{\infty, 1}
         }
         \\
         &= H \norm{\l - \tilde \l}_\infty + 
          	H^2 \norm{P - \tilde P}_{\infty, 1}
    \end{align*}
\end{proof}
 
 \begin{lemma}[action-value-difference]\label{lem:Q_diff}
	Let $M = (S, \A_i, H, P, \l), \tilde M = (S, \A_i, H, \tilde P, \tilde \l)$ be two MDPs, and $\pi, \tilde \pi \in \S \to \Delta_{\A_i}$ be a pair of policies.
	Then for all $s\in \S, a\in \A_i$, we have;
    \begin{align*}
            Q^{\pi} (s, a; M) - Q^{\tilde \pi} (s, a; \tilde M)
            &\leq H^2 \norm{\pi - \tilde \pi}_{\infty, 1}
            + (H^2+1) \norm{P - \tilde P}_{\infty, 1}
            + (H+1)\norm{\l - \tilde \l}_\infty
            .
        \end{align*}
    \end{lemma}
    \begin{proof}
        By \cref{lem:val_diff}, we have
        \begin{align*}
            V^\pi(s) - \tilde V^{\tilde \pi}(s)
            &= V^\pi(s) - V^{\tilde \pi}(s)
            + V^{\tilde \pi}(s) - \tilde V^{\tilde \pi}(s)
            \\
            &\leq H^2 \norm{\pi - \tilde \pi}_{\infty, 1}
                + H^2 \norm{P - \tilde P}_{\infty, 1}
                + H\norm{\l - \tilde \l}_\infty.
        \end{align*}  
		Thus, let $s\in \S_h, a\in \A_i$, and observe;
        \begin{align*}
            Q^{\pi}(s, a; M) - Q^{\tilde \pi}_{h}(s, a; \tilde M)
            &= \l(s, a) - \tilde \l(s, a)
            \\
            &\quad + \E_{s' \sim P(\cdot \mid s, a)} V^\pi(s'; M)
            - \E_{s' \sim \tilde P(\cdot \mid s, a)} V^{\tilde \pi}(s'; \tilde M)
            \\
            &= \l(s, a) - \tilde \l(s, a)
            \\
            &\quad + \sum_{s' \in \S_{h+1}}
                P(s' \mid s, a) V^\pi(s'; M)
                - \tilde P(s' \mid s, a) V^{\tilde \pi}(s'; \tilde M)
            \\
            &\leq |\l(s, a) - \tilde \l(s, a)|
            \\
            &\quad + \sum_{s' \in \S_{h+1}}
                P(s' \mid s, a) 
                |V^\pi(s'; M)
                - V^{\tilde \pi}(s'; \tilde M)|
            \\
            &\quad + \sum_{s' \in \S_{h+1}}
                V^{\tilde \pi}(s'; \tilde M)
                |P(s' \mid s, a) 
                - \tilde P(s' \mid s, a)|
            \\
            &\leq |\l(s, a) - \tilde \l(s, a)|
            \\
            &\quad + H^2 \norm{\pi - \tilde \pi}_{\infty, 1}
            + H^2 \norm{P - \tilde P}_{\infty, 1}
            + H\norm{\l - \tilde \l}_\infty
            \\
            &\quad + H\norm{
                P(\cdot \mid s, a) 
                - \tilde P(\cdot \mid s, a)
            }_1
            \\
            &\leq H^2 \norm{\pi - \tilde \pi}_{\infty, 1}
            + (H^2 +1)\norm{P - \tilde P}_{\infty, 1}
            + (H+1)\norm{\l - \tilde \l}_\infty
        \end{align*}

    \end{proof}

\begin{lemma}\label{lem:occupancy_path}
    For any policy $\mu \colon \S \to \A_i$, player $i\in[m]$, we have
    \begin{align*}
        \normb{q^\mu_{P^i_{t+1}} - q^\mu_{P^i_{t}}}_\infty
        &\leq H^2 \sum_{j\neq i}
            \norms{\pi^{j}_{t+1} - \pi^{j}_{t}}_{\infty,1}
        .
    \end{align*}
\end{lemma}
\begin{proof}
    Follows by combining \cref{lem:occupancy_variation} and \cref{lem:mdp_variation}; 
    \begin{align*}
        \normb{q^\mu_{P^i_{t+1}} - q^\mu_{P^i_{t}}}_1
        \leq H^2 \normb{
            P_{t+1}^i - P_t^i
        }_{\infty, 1}
        \leq H^2 \sum_{j\neq i}
            \norms{\pi^{j}_{t+1} 
        - \pi^{j}_{t}}_{\infty, 1}
        .
    \end{align*}
\end{proof}

\begin{lemma}\label{lem:mdp_variation}
    It holds that for all $i\in[m]$, $s\in \S, a\in \A_i$;
    \begin{align*}
        \norms{P_{t+1}^i(\cdot \mid s, a) - P_t^i(\cdot \mid s, a)}_1 
        &\leq \sum_{j\neq i}
            \norms{\pi^{j}_{t+1}(\cdot \mid s) 
        - \pi^{j}_{t}(\cdot \mid s)}_1,
        \\
        |\l_{t+1}^i(s, a)
        - \l_t^i(s, a)|
        &\leq \sum_{j\neq i}
            \norms{\pi^{j}_{t+1}(\cdot \mid s) 
        - \pi^{j}_{t}(\cdot \mid s)}_1
        .
    \end{align*}
\end{lemma}
\begin{proof}
    For the losses, observe;
    \begin{align*}
        \l_{t+1}^i(s, a)
        - \l_{t}^i (s, a)
        &= \E_{\ba^{-i} \sim \bpi_{t+1}^{-i}} \l^i(s, a, \ba^{-i})
            - \E_{\ba^{-i} \sim \bpi_{t}^{-i}} \l^i(s, a, \ba^{-i})
        \\
        &= \sum_{\ba^{-i} \in \A^{-i}} \b{
                \bpi^{-i}_{t+1}(\ba^{-i} \mid s) 
                - \bpi^{-i}_{t}(\ba^{-i} \mid s)
            } \l^i(s, a, \ba^{-i})
        \\
        &\leq \norm{\bpi^{-i}_{t+1}(\cdot \mid s) 
        - \bpi^{-i}_{t}(\cdot \mid s)}_1.
    \end{align*}
    For the induced transition function, note that for any $h\in[H]$, we have
    \begin{align*}
        \sum_{s'\in \S_{h+1}}
            & P^i_{t+1}(s' \mid s, a) 
            - P^i_{t} (s' \mid s, a)
        \\
        &= \sum_{s'\in \S_{h+1}} 
            \E_{\ba^{-i} \sim \bpi_{t+1}^{-i}
                (\cdot \mid s)}\sb{P(s' \mid s, a, \ba^{-i})}
            - \E_{\ba^{-i} \sim \pi_t^{-i}
                (\cdot \mid s)}\sb{P(s' \mid s, a, \ba^{-i})}
        \\
        &= \sum_{s'\in \S_{h+1}} \sum_{\ba^{-i}\in \A^{-i}} 
            (\bpi^{-i}_{t+1}(\ba^{-i} \mid s) 
            - \bpi^{-i}_{t}(\ba^{-i} \mid s) ) 
            P(s' \mid s, a, \ba^{-i})
        \\
        &= \sum_{ \ba^{-i} \in \A^{-i}} 
            (\bpi^{-i}_{t+1}( \ba^{-i} \mid s) 
            - \bpi^{-i}_{t}( \ba^{-i} \mid s) ) 
            \sum_{s'\in \S_{h+1}} P(s' \mid s, a, \ba^{-i})
        \\
        &= \norm{\bpi^{-i}_{t+1}(\cdot \mid s) 
        - \bpi^{-i}_{t}(\cdot \mid s)}_1
    \end{align*}
    By \cref{lem:l1_prod}, we have
    \begin{align*}
        \norm{\bpi^{-i}_{t+1}(\cdot \mid s) 
        - \bpi^{-i}_{t}(\cdot \mid s)}_1
        \leq \sum_{j\neq i}
            \norms{\pi^{j}_{t+1}(\cdot \mid s) 
        - \pi^{j}_{t}(\cdot \mid s)}_1,
    \end{align*}
    and the result follows.
\end{proof}

\begin{lemma}\label{lem:occupancy_variation}
    For any policy $\pi\in \S \to \A_i$ and single agent transition functions $P, \tilde P$, it holds that
    \begin{align*}
        \norms{q_{P}^\pi - q_{ \tilde P}^\pi}_1
        &\leq H^2 \norms{P - \tilde P}_{\infty, 1},
        \\
        \norms{q_{P}^\pi - q_{ \tilde P}^\pi}_\infty
        &\leq H \norms{P - \tilde P}_{\infty, 1}.
    \end{align*}
\end{lemma}
\begin{proof}
    Let $L \in [H]$, $z \in \S_L$, set loss function 
    $\l_z(s, a) = \mathbb I\{ s = z\}$, and consider the two MDPs
    $M_z=(H, \S, \A_i, P, \l_z)$ and 
    $\tilde M_z=(H, \S, \A_i, \tilde P, \tilde \l_z)$ with value functions $V_z, \tilde V_z$ respectively.
    Then, we have for any $s\in \S_h$,
    $V^\pi_z(s) = \Pr(s_L = z \mid s_h = s, P, \pi)$, and 
    $\tilde V_z(s) = \Pr(s_L = z \mid s_h = s, \tilde P, \pi)$, which also implies 
    $V^\pi_z(s_1) = q^{\pi}_P(z)$ and
    $\tilde V^\pi_z(s_1) = q^{\pi}_{\tilde P}(z)$. 
    Thus, by \cref{lem:val_diff}, we have;
    \begin{align*}
        q_P^\pi(z) & - q_{\tilde P}^\pi(z)
        \\
        &= \sum_{h=1}^L \sum_{s_h, a_h}
            q_{\tilde P}^\pi(s_h, a_h)
            \sum_{s_{h+1}} 
                (P^\pi(s_{h+1} \mid s_h, a_h)
                - \tilde P^\pi(s_{h+1} \mid s_h, a_h)
                )\Pr(s_L = z \mid s_{h+1}, P, \pi).
    \end{align*}
    Taking absolute values and summing the above over $z\in \S_L$ we obtain
    \begin{align*}
        \sum_{z\in \S_L} | q_P^\pi(z) - q_{\tilde P}^\pi(z) |
        &\leq \sum_{h=1}^L \sum_{s_h, a_h}
            q_{\tilde P}^\pi(s_h, a_h)
            \sum_{s_{h+1}} 
                |P^\pi(s_{h+1} \mid s_h, a_h)
                - \tilde P^\pi(s_{h+1} \mid s_h, a_h)
                |
        \\
        &\leq L \norms{ P - \tilde P }_{\infty, 1}.
    \end{align*}
    Hence,
    \begin{align*}
        \norms{q_P^\pi - q_{\tilde P}^\pi}_1
        = \sum_{L=1}^H \sum_{z\in \S_L} 
            |q_P^\pi(z) - q_{\tilde P}^\pi(z)|
        \leq H^2 \norms{P - \tilde P}_{\infty, 1}.
    \end{align*}
    
\end{proof}

\section{FTRL lower bound in non-stationary MDP}
\label{sec:ftrl_nonstationary_lb}
In the following, we provide an example demonstrating that FTRL-based policy optimization does not adapt to non-stationary dynamics, at least not in the sense discussed here. 

In a nutshell, since FTRL considers the entire sequence of past loss functions, it may not pick up on  the change in the long term reward in a timely fashion.
Indeed, since the policy optimization paradigm prescribes a per state objective that effectively ignores the visitation frequency to that state, FTRL allows past losses (induced by action-value functions from previous episodes) \emph{that may be irrelevant} to bias the policy towards suboptimal actions for a prohibitively large number of episodes. Loosely speaking, this behavior is due to the fact that in contrast to OMD, FTRL is insensitive to the order of losses. 

Notably, the failure of FTRL is strongly related to its inability to guarantee \emph{adaptive} regret in the sense defined in \cite{hazan2009efficient}, who also point out the inherent non-adaptivity of this algorithm. 

The claim below illustrates an example of an MDP with a small constant change in the dynamics leading to FTRL incurring linear regret. 
Essentially, this is a simple example where FTRL fails to achieve adaptive regret, embedded in a non-stationary MDP.
We remark that while the MDP in our construction makes a single, abrupt shift in the dynamics, the lower bound does not stem from the abruptness of the change. Rather, this choice is only for simplicity; the construction may be generalized to the case where the per episode drift must be bounded by e.g., 
$1/\sqrt T$, by augmenting the construction with a "shift period" of $\sqrt T$ episodes. In addition, it is not hard to show the construction can be generalized to subsets of the action simplex --- this is to say that the lower bound also does not stem from lack of exploration that can be solved by truncating the simplex, as we have done in the OMD case. Finally, we remark that the same lower bound remains valid also when considering OFTRL; 1-step recency bias does not make the algorithm sufficiently adaptive for the example in question.

We refer in the statement to a symmetric regularizer, meaning one that is insensitive to permutations of the input coordinates.
This assumption is only for simplicity; it can be relaxed by generalizing the instance appearing in the lower bound to a mixture of two instances with action roles reversed, and observing that on at least one of them FTRL must incur linear regret.

\begin{claim}
	There exists a non-stationary MDP $M = (S, \{a, b\}, \{P_t\}_{t=1}^T, \l )$, such that $\sum_{t=2}^T \norms{P_t - P_{t-1}}_1 \leq 1$, but policy optimization with FTRL over the action simplex $\Delta_{\cb{a, b}}$, any symmetric regularizer, and any step size incurs regret of $\Omega(T)$.
\end{claim}
\begin{proof}
	Let $S=\cb{s_0, s_1, s_2, L_0, L_1}$ denote the state space, and consider the non-stationary MDP $M = (S, \{a, b\}, \{P_t\}_{t=1}^T, \l )$, where the immediate loss function is independent of the action and is specified by $\l(s_i)=0\; \forall i$, $\l(L_0) = 0$, and $\l(L_1) = 1$.
	Further, assume that
	\begin{itemize}
		\item for $t\leq T/3$, $P_t(s_1|s_0, \cdot) = 1, P_t(L_1|s_2, a) = 1$, 
			and $P_t(L_0|s_2, b) = 1$, (see \cref{fig:ftrl_lb_1})
		\item for $t > T/3$, $P_t(s_2|s_0, \cdot) = 1, P_t(L_1|s_2, b) = 1$, 
			and $P_t(L_0|s_2, a) = 1$ (see \cref{fig:ftrl_lb_2}).
	\end{itemize}
	Consider running policy optimization with FTRL over $\X = \Delta_{\{a, b\}}$, and a symmetric regularizer $R \colon \X \to \R$ for $T$ episodes.
	
	First, observe that the optimal policy in hindsight $\pi_\star$ selects action $a$ with probability $1$ in state $s_2$; $\pi_\star(a|s_2) = 1$. Note that the actions chosen in the rest of the states do not affect the loss, and therefore need not be specified. Thus, in the first $T/3$ episodes, $\pi_\star$ loses nothing since state $s_2$ is never reached, and in the remaining $2T/3$ episodes it loses nothing on account of selecting an action which leads to $L_0$. This establishes that $\sum_{t=1}^T V^{\pi_\star}(s_0; P_t) = 0$. 
	
	On the other hand, for $t > T/3$, the FTRL objective on episode $t$ at state $s_2$, is given by
	\begin{align*}
		\pi_{t+1}(\cdot|s_2) \gets \argmin_{x\in \X} \cb[B]{
			\eta \ab[B]{\sum_{j=1}^t \hat \l_t, x} + R(x)
		},
	\end{align*}
	where $\hat \l_t(\cdot) \eqq Q_t(s_2, \cdot)$. Thus,
	\begin{align*}
		\begin{cases}
			\hat \l_t(a) = 1, \hat \l_t(b) = 0
				\quad &t\leq T/3,
			\\
			\hat \l_t(a) = 0, \hat \l_t(b) = 1
				\quad &t > T/3,
		\end{cases}
	\end{align*}
	which leads us to conclude that in all rounds $t \leq 2T/3$, the action $b$ seems favorable according to the minimization objective. This implies that for all $t \leq 2T/3$, $\pi_t(b|s_2) \geq 1/2$. Note that we use here the fact that the regularizer and decision set are symmetric, and treat all coordinates equally. Now,
	\begin{align*}
		\sum_{t=1}^T V^{\pi_t}(s_0; P_t) - V^{\pi_\star}(s_0; P_t) 
			\geq \sum_{t=T/3}^{2T/3} V^{\pi_t}(s_0; P_t)
			\geq \frac{T}{6},
	\end{align*}
	as claimed.
\end{proof}

\begin{figure}[!p]
    \centering
\tikzset{arrow/.style={-stealth, thick, draw=gray!80!black}}

\begin{tikzpicture}[
    state/.style = {
        circle, draw, inner sep=12pt
    },
]
	\node[state] at (3.5,7.5) (G0) {$L_0$};
	\node[state] at (6.5,7.5) (G1) {$L_1$};
	\node[state] at (5,5) (S2) {$s_2$};
	\node[state] at (2.5,2.5) (S0) {$s_0$};
	\node[state] at (0,5) (S1) {$s_1$};
	
	\draw [arrow] (S2) edge node[left,pos=0.47]{$a$} (G1);
	\draw [arrow] (S2) edge node[right]{$b$} (G0);
	
    \draw[dashed, ->] (S0) edge node[left]{} (S2);
    \draw [arrow] (S0) edge node[right]{$a,b$} (S1);
\end{tikzpicture}

\caption{MDP at $t=1, \ldots, T/3$} 
\label{fig:ftrl_lb_1}

\end{figure}
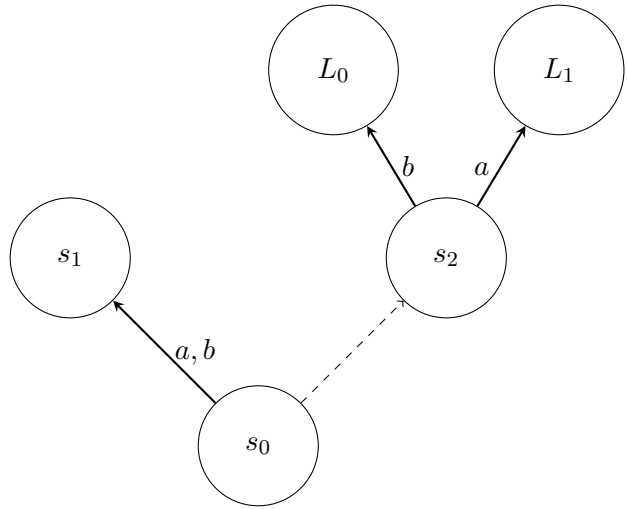

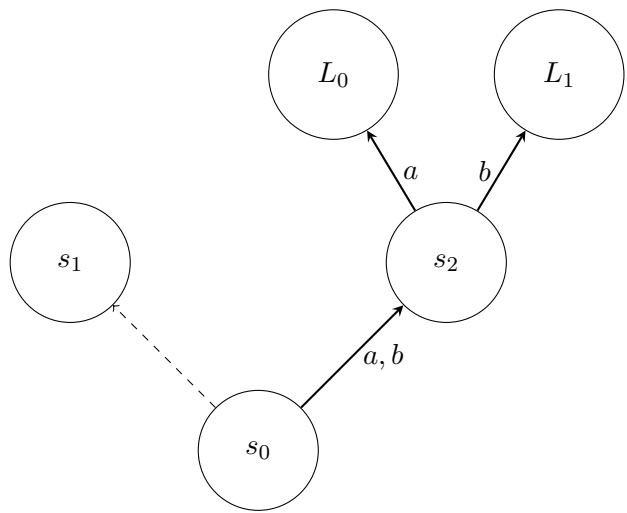
\begin{figure}
    \centering
\tikzset{arrow/.style={-stealth, thick, draw=gray!80!black}}

\begin{tikzpicture}[
    state/.style = {
        circle, draw, inner sep=12pt
    },
]
	\node[state] at (3.5,7.5) (G0) {$L_0$};
	\node[state] at (6.5,7.5) (G1) {$L_1$};
	\node[state] at (5,5) (S2) {$s_2$};
	\node[state] at (2.5,2.5) (S0) {$s_0$};
	\node[state] at (0,5) (S1) {$s_1$};
	
	\draw [arrow] (S2) edge node[left]{$b$} (G1);
	\draw [arrow] (S2) edge node[right,pos=0.47]{$a$} (G0);
	
    \draw[dashed, ->] (S0) edge node[left]{} (S1);
    \draw [arrow] (S0) edge node[right]{$a,b$} (S2);
\end{tikzpicture}

\caption{MDP at $t=T/3+1, \ldots, T$} 
\label{fig:ftrl_lb_2}

\end{figure}

\section{Auxiliary Lemmas}
 
\begin{lemma}\label{lem:logbarrier_simple_props}
    Let $k \in \mathbb N$, and consider the truncated simplex $\Delta_k^\gamma \subseteq \Delta_k$ (see \cref{eq:trunc_simplex}). It holds that:
    \begin{enumerate}
        \item For log-barrier regularizer $R\colon \Delta_k^\gamma \to \R$ (see \cref{eq:log_barrier}), we have $D_R(x, x') \leq \frac{3}{\gamma} $ for all $x, x'\in \Delta_k^\gamma$.
        \item If $0 < \gamma \leq 1/2k$, for all $x\in \Delta_k$, there exists $x^\gamma \in \Delta_k^\gamma$ such that
        $\norm{x - x^\gamma}_1 \leq 2 \gamma k$.
    \end{enumerate}
\end{lemma}
\begin{proof}
    See below.
    \begin{itemize}
        \item 
            We have, for any $x, x' \in \Delta_{\A_i}^\gamma$;
        \begin{align*}
            D_R(x, x') 
            = \sum_{a\in \A_i}\log\frac{x(a)}{x'(a)}
              + \frac{x(a) - x'(a)}{x'(a)}
            \leq \log\frac{1}{\gamma} + \frac{2}{\gamma}
            \leq \frac{3}{\gamma}.
        \end{align*}
        
        \item Let $I = \{i\in [k] \mid x(a) \leq \gamma \}$. Then $|I| \leq k-1$, otherwise $\sum_{i=1}^k x(i) \leq 1/2$. Now, set $\tilde x^\gamma(i) = \gamma$, for $i\in I$, and $\tilde x^\gamma(i) = x(i)$ for $i\notin I$. We have
        \begin{align*}
            \sum_{i=1}^k \tilde x^\gamma(i) = 
            1 + \delta, 
            \quad \text{where } \delta \leq \gamma |I|,
        \end{align*}
        and $\norm{x - \tilde x^\gamma}_1 \leq (k-1)\gamma$.
        Now, subtract from the largest coordinate value $\tilde x^\gamma(i_{\rm max})$ the excess weight $\delta$. In the event that $\tilde x^\gamma(i_{\rm max}) \leq \gamma + \delta$, subtract to $\gamma$, and continue iteratively to the second largest etc. This process must terminate before reaching coordinates in $I$, since $\sum_{i=1}^k x(i) = 1$. Now, $\norm{x^\gamma}_1 = 1$, and
        \begin{align*}
            \norm{x - x^\gamma}_1
            \leq (k-1)\gamma + \delta
            \leq (2k - 1) \gamma.
        \end{align*}
    \end{itemize}
    
\end{proof}

\begin{lemma}\label{lem:l1_prod}
    Let $p$ and $q$ be any two product distributions over $X_1 \times \cdots \times X_m$, i.e., $p(x_1, \ldots, x_m) = \prod_{i=1}^m p_i(x_i)$, and $q(x_1, \ldots, x_m) = \prod_{i=1}^m q_i(x_i)$. Then
    \begin{align*}
        \norm{p - q}_1
        \leq \sum_{i=1}^m \norm{p_i - q_i}_1.
    \end{align*}
\end{lemma}
\begin{proof}
    We have;
    \begin{align*}
        \norm{p - q}_1
        &= \sum_{x_1 \in X_1} \cdots \sum_{x_m \in X_m} 
            \Big| p_m(x_m) \prod_{i=1}^{m-1} p_i(x_i)
            - q_m(x_m) \prod_{i=1}^{m-1} q_i(x_i)
            \Big|
        \\
        &\leq \sum_{x_1 \in X_1} \cdots \sum_{x_m \in X_m} 
        \prod_{i=1}^{m-1} p_i(x_i) \Big|
            p_m(x_m) - q_m(x_m)
        \Big|
        \\
        &\quad + \sum_{x_1 \in X_1} \cdots \sum_{x_m \in X_m} 
            q_m(x_m) \Big|
            \prod_{i=1}^{m-1} p_i(x_i) 
            - \prod_{i=1}^{m-1} q_i(x_i)
        \Big|
        \\
        &=  \norm{
            p_m - q_m
        }_1 \sum_{x_1 \in X_1} \cdots \sum_{x_{m-1} \in X_{m-1}} 
        \prod_{i=1}^{m-1} p_i(x_i) 
        \\
        &\quad + \norm{
                \prod_{i=1}^{m-1} p_i
                - \prod_{i=1}^{m-1} q_i
            }_1 \sum_{x_m \in X_m} 
            q_m(x_m) 
        \\&= \norm{
            p_m - q_m
        }_1 + \norm{
            \prod_{i=1}^{m-1} p_i
            - \prod_{i=1}^{m-1} q_i
        }_1,
    \end{align*}
    and the claim follows by induction.
\end{proof}

\section{Markov Games with Independent Transition Function }
\label{sec:quasi_markov}
In this section we consider a variant of Markov Games for which each agent has its own state and the transition is affected only by the agent's own action. Formally, each agent has its own set of states $\S^i$. Further, $P$ is the transition kernel, where
given the state at time $h$, $s\in\S^i_h$, and the agent's action $a\in\A_i$, 
$P(\cdot \mid s, a) \in \Delta_{\S^i_{h+1}}$ is  the probability distribution over the next state. The loss function at time $h$ depends on the states and actions at time $h$ of all agents: $\ell_h^i:(\bigtimes_{i\in[m]}\S^i_h)\times \A \to [0,1]$ The policy of player $i$, depends on its individual state. That is, $\pi^i(\cdot \mid \cdot): \A_i \times \S^i \to [0,1]$, is a function such that $\pi^i(a\mid s)$ gives the probability of player $i$ to take action $a$ in state $s$. Similar to before, denote the expected loss function of agent $i$ at time $t$ given action $a$ and state $s\in\S_h^i$ by $\ell_t^i(s,a) = \E[\ell^i(\bs,\ba)\mid \bpi_t , s_h^i = s]$ where $\bpi_t$ is the joint policy of the agents and $\bs = (s_h^1,...,s_h^m)$ is the vector of the agents' states at time $h$. 	Similar to before, we denote the value and action-value functions of a policy $\pi \in \S \to \Delta_{\A_i}$ by
\begin{align*}
    V_t^{i,\pi}(s) = \E \sb{ \sum_{h' = h}^H \ell_t^i(\bs_{h'},\ba_{h'}) 
    	\mid \bpi_t^{-i}, s_h = s}
    \; ; \;
    Q_t^{i,\pi}(s, a) = \E \sb{ \sum_{h' = h}^H \ell_t^i(\bs_{h'},\ba_{h'}) 
    	\mid \bpi_t^{-i}, s_h = s},
\end{align*}
	 where $s \in \S^i_h$ and $a\in \A_i$. We note that we sometimes use the shorthand $V_t^i(\cdot)$ for $V_t^{i,\pi^i_t}(\cdot)$ and $Q_t^i(\cdot,\cdot)$ for $Q_t^{i,\pi^i_t}(\cdot,\cdot)$.

\paragraph{}
In the setting of individual state transitions, it is possible to achieve much better regret bounds than in Markov games; specifically, we show that using \cref{alg:POSR} each player can obtain $O(\log T)$ individual swap regret. This possibility stems from the fact that in contrast to Markov games, the MDPs each player experiences throughout the episodes remain constant, and hence it is possible to obtain a regret bound which depends on the sum of second order path lengths of the players' policies rather than on the first order path lengths (see \cref{thm:master_regret_by_path} for the corresponding result for Markov games).

\begin{theorem}
\label{thm:regret-by-path-quasi}
    In the independent transition function setting, assume that every player $i$ adopts \cref{alg:POSR} with log-barrier regularization (\cref{eq:log_barrier}) and $\gamma\leq 1/2 A_i$, and that $\norms{\hat Q_t - Q_t}_\infty \leq \eps$ for all $t$.
    Then, assuming $H \geq 2$, the swap-regret of player $i$ is bounded as
    \begin{align*}
        \SwRe_T^i  \leq \frac{A^{2}\log\frac{1}{\gamma}}{\eta} + 24 \eta H^{4}Am\sum_{j=1}^{m}\sum_{t=1}^{T}\norm{\pi_{t+1}^{j}-\pi_{t}^{j}}_{\infty,1}^{2}+\eps HT + 8 \eta \eps^{2}T
    \end{align*}
\end{theorem}

\begin{proof}
    As opposed to the standard Markov game setting, the occupancy measure of the benchmark policy remains stationary over time: $q_t^{i,\star} = q^{i,\star}$.
    Therefore, much like in the proof of \cref{thm:master_regret_by_path},
    \begin{align}
        \nonumber
        \Re_T^i(\pi_\star^i)
        & \leq  \eps H T + \sum_{s\in\S}
        	q^{i,\star}(s) \sum_{t=1}^{T} \ab{
        \hat{Q}_t^i(s,\cdot) ,
        \pi_{t}^{i}(\cdot\mid s) - x_{\star}^{i, s}
        }
        \\
        \nonumber
        & \leq \eps H T +   \sum_{s\in\S} q^{i,\star}(s) \sum_{a\in \A_i} \sum_{t=1}^T 
            \ab{g_t^{i,s,a}, x_t^{i, s, a} - x_\star^{i, s}}
        .
    \label{eq:quasi-proof-regret-by-path:value-diff}
    \end{align}
    From \cref{lem:rvu-base},
    \begin{align*}
        \sum_{t=1}^T 
            \ab{g_t^{i,s,a}, x_t^{i, s, a} - x_\star^{i, s}} &\leq \frac{A\log\frac{1}{\gamma}}{\eta} + 4\sum_{t=1}^T\eta\norm{g_{t}^{i,s,a} - g_{t-1}^{i,s,a}}^{2}_{*, x_t^{i,s,a}}
            \\
            & \leq
            \frac{A\log\frac{1}{\gamma}}{\eta} + 24\eta H^{4}m\sum_{j=1}^{m}\sum_{t=1}^{T}\norm{\pi_{t+1}^{j}-\pi_{t}^{j}}_{\infty,1}^{2} + 8 \eta \eps^{2}\sum_{t=1}^{T}\pi_{t+1}^{j}(a\mid s).
    \end{align*}
    Combining the last two displays completes the proof.
\end{proof}

\begin{theorem} \label{thm:path length-apprx-quazi}
    If each player uses \cref{alg:POSR} with log-barrier regularization (\cref{eq:log_barrier}) and $\eta = \frac{1}{96 H^2 m\sqrt{SA}}$  then the following path length bound holds on the jointly generated policy sequence;
    \begin{align*}
        \sum_{t=1}^{T} \sum_{i=1}^{m} \norm{\pi_{t+1}^{i} - \pi_{t}^{i}}_{\infty,1}^{2} 
        & \leq 
        768 S A^{3} m \log \frac{1}{\gamma} + \frac{4 \eps^{2}T}{mH^{4}}.
    \end{align*}
\end{theorem}
The proof follows by the exact same arguments in the proof of \cref{thm:path length-apprx}. Combining \cref{thm:regret-by-path-quasi,thm:path length-apprx-quazi} gives us the following corollary:
\begin{corollary}
\label{cor:independent_dynamics}
    In the independent transition function setting with full information (i.e., $\eps = 0$), assume that every player $i$ adopts \cref{alg:POSR} with log-barrier regularization (\cref{eq:log_barrier}), $\eta = \frac{1}{96 H^2 m\sqrt{SA}}$ and $\gamma = 1/T$.
    Then, assuming $H \geq 2$ and $T\geq 2 A$, the swap-regret of player $i$ is bounded as
    \begin{align*}
        \SwRe_T^i  
        \leq 
        288  H^{2} S^{3/2} A^{7/2} m \log T.
    \end{align*}
\end{corollary}

\end{document}